\documentclass[nohyperref]{article}

\usepackage{microtype}
\usepackage{graphicx}
\usepackage{subfigure}
\usepackage{booktabs}

\usepackage[preprint,nohyperref]{icml2022}

\usepackage{amsmath}
\usepackage{amssymb}
\usepackage{mathtools}
\usepackage{amsthm}

\theoremstyle{plain}
\newtheorem{theorem}{Theorem}[section]

\newtheorem{lemma}[theorem]{Lemma}
\newtheorem{corollary}[theorem]{Corollary}
\theoremstyle{definition}
\newtheorem{definition}[theorem]{Definition}

\theoremstyle{remark}

\usepackage[textsize=tiny]{todonotes}

\definecolor{orange}{RGB}{255, 165, 0}

\newcommand{\eat}[1]{}

\newcommand{\ONESHOT}{\textsc{oneshot}}
\newcommand{\JOINT}{\textsc{joint}}
\newcommand{\PEELING}{\textsc{peeling}}
\newcommand{\CANONICAL}{\textsc{canonical}}

\usepackage[algo2e,ruled, vlined]{algorithm2e}
\usepackage{graphicx}
\usepackage{subfigure}
\usepackage{booktabs}
\usepackage{tikz}
\usetikzlibrary{patterns}
\usepackage{pgfplots}
\usepackage{mathtools}
\usepackage{float}
\usepackage{mathrsfs,amsmath}
\usepackage{stackengine} 
\usepackage{colortbl}

\usepackage{listings}

\usetikzlibrary{matrix,positioning}

\usepgfplotslibrary{patchplots}

\newcommand{\TOP}[1]{\ensuremath{\textsc{top}\text{-}#1\xspace}}
\newcommand{\GOOD}[1]{\ensuremath{\textsc{good}\text{-}#1\xspace}}

\DeclareMathOperator*{\argmax}{arg\,max}
\DeclareMathOperator*{\argmin}{arg\,min}

\usepackage{textcomp}
\usepackage{xcolor}

\renewcommand{\ln}{\log}
\newcommand{\eps}{\varepsilon}

\newcommand{\OPT}{\textsc{opt}}
\newcommand{\y}{y}

\newcommand{\kargmax}[2]{{\argmax_{#2}}^{[#1]}}

\DeclareMathOperator\sgn{sgn}

\usepgfplotslibrary{fillbetween}

\newcommand{\loss}{\textsc{loss}}

\newcommand{\topitem}[1]{ j_{#1}}
\renewcommand{\top}[1]{\ensuremath\vec{x}_{[#1]}}

\newcommand{\Sidak}{Šidák}
\newcommand{\Bonferroni}{Bonferroni}
\newcommand{\LHopitalsrule}{L'Hôpital's rule}

\icmltitlerunning{Differentially Private Top-k Selection via Canonical Lipschitz Mechanism}

\begin{document}

\twocolumn[

\icmltitle{Differentially Private Top-k Selection via Canonical Lipschitz Mechanism}

\icmlsetsymbol{equal}{*}

\begin{icmlauthorlist}
\icmlauthor{Michael Shekelyan}{kcl}
\icmlauthor{Grigorios Loukides}{kcl}
\end{icmlauthorlist}

\icmlaffiliation{kcl}{King's College London, Department of Informatics, London, United Kingdom}

\icmlcorrespondingauthor{Michael Shekelyan}{michael.shekelyan@kcl.co.uk}

\icmlkeywords{Privacy-preserving Statistics, Differential Privacy}

\vskip 0.3in
]

\printAffiliationsAndNotice{}

\newcommand{\SINGLEDOUBLE}[2]{#2}
\newcommand{\REVISE}[2]{#2}
\newcommand{\FORK}[2]{#1}

\begin{abstract}
Selecting the top-$k$ highest scoring items under differential privacy (DP) is a fundamental task with many applications. This work presents three new results. First, the exponential mechanism, permute-and-flip and report-noisy-max, as well as their oneshot variants, are unified into the Lipschitz mechanism, an additive noise mechanism with a single DP-proof via a mandated Lipschitz property for the noise distribution. Second, this new generalized mechanism is paired with a canonical loss  function to obtain the canonical Lipschitz mechanism, which can directly select k-subsets out of $d$ items in $O(dk+d \log d)$ time. The canonical loss function assesses subsets by how many users must change for the subset to become top-$k$. Third, this composition-free approach to subset selection  improves utility guarantees by an $\Omega(\log k)$ factor compared to one-by-one selection via sequential composition, and our experiments on synthetic and real-world data indicate substantial utility improvements.
\end{abstract}

\section{Introduction} 

Let $\{1, \ldots, d\}$ be a set of items and  $\vec{x} \in \mathbb{R}^d$ be a data vector comprised of \REVISE{the numerical scores of these items}{numerical scores for these items}. Depending on the application domain the items can be thought of as features, policies, models, or in some cases physical objects as the term suggests. The top-$k$ selection problem seeks to select $k$ highest scoring items, i.e.,  $\kargmax{k}{i \in \{1, \ldots, d\}} \{ \vec{x}_i \}$. It is a fundamental problem with myriads of applications (see~\cite{top_k_survey} for a survey) and also a building block in analytic tasks including classification, summarization, and content extraction~\cite{kais_top10,aaai_13}. The applications of top-$k$ selection are typically fueled by \REVISE{users' data}{user data} and thus have raised privacy  concerns~\cite{netflix_deanon}. 
To address such concerns, it is crucial that top-$k$ selection preserves privacy. This is possible by enforcing differential privacy (DP) ~\cite{dwork2014algorithmic} which, informally speaking, ensures that the selected items do not depend heavily on any \REVISE{}{single }user's private information. 

A data vector $\vec{x}$ is derived from some object $\hat{x} \in \mathbb{X}$  and is influenced through the private information of a set  of users $\textsc{users}(\hat{x})$. This influence is a central concept in DP, which deploys random processes (mechanisms) whose behavior is relatively similar with or without the input of a particular user. To achieve this, functions which determine the behavior of the process must have a limited sensitivity to individual users:
\begin{definition}[sensitivity] \label{def:sensitivity}
Let $\mathbb{Y}$ be some arbitrary domain. A function $f : \mathbb{Y} \times \mathbb{X} \rightarrow \mathbb{R}$ has sensitivity $\Delta_f \in \mathbb{R}_{\ge 0}$, if $|f( y \mid \hat{b)}-f( y \mid \hat{a})| \le \Delta_f$, for any $y \in \mathbb{Y}$ and any pair of   objects $\hat{a}, \hat{b} \in \mathbb{X}$ with Hamming distance 1 between $\textsc{users}(\hat{a})$ and $\textsc{users}(\hat{b})$.
\end{definition}

The methods proposed in this work are oblivious\footnote{While the underlying object $\hat{x}$ can often be thought of as a database of user records, it may also lack a natural division into user-specific parts. For example, $\vec{x}$ can contain numerical features of some video $\hat{x}$ in which the individuals $\textsc{users}(\hat{x})$ appear in. In this case, there are no ``user records'', but one can still model each individual as a ``user'' who  influences the numerical features $\vec{x}$ .} to $\textsc{users}(\hat{x})$ or even $|\textsc{users}(\hat{x})|$. The notation $f(y \mid \hat{x})$ conveys that, given a fixed $\hat{x}$, the function $f$ behaves like a regular function $\mathbb{Y} \rightarrow \mathbb{R}$. 

In practice the sensitivity of functions may be asymmetric ($-\Delta_1,+\Delta_2$). Yet, for shift-invariant mechanisms \cite{mckenna2020permute} one can apply a shifting trick to reduce it to a sensitivity $\frac{|\Delta_2-\Delta_1|}{2}$ (cf. Theorem~\ref{thm:asymmetric} in Appendix~\ref{sec:proofs}) generalizing  methods for \REVISE{count-based}{monotonic} functions \cite{dwork2013algorithmic,mckenna2020permute} \REVISE{  without a need for special methods \cite{durfee2019practical}}{without having to resort to generalized sensitivity notions like bounded range \cite{durfee2019practical}. }

The $\eps$-DP methods translate the bounded sensitivity over the function values into a bounded sensitivity $\eps \in \mathbb{R}_{\ge0}$ over the log-probabilities of selection options:

\begin{definition}[$\eps$-DP selection]
Let $Y$ be a random variable supported over a set of possible outputs $\mathbb{Y}$ and $y \in \mathbb{Y}$. Reporting the value $Y = y$ is $\eps$-DP, for a given $\eps \in \mathbb{R}_{\ge 0}$, if $f(y \mid \hat{x}) = \log Pr[\,Y = y \mid \hat{x}\,]$ has sensitivity $\Delta_f = \eps$.
\end{definition}

\begin{figure}
\SINGLEDOUBLE{\centering\includegraphics[width=0.75\textwidth]{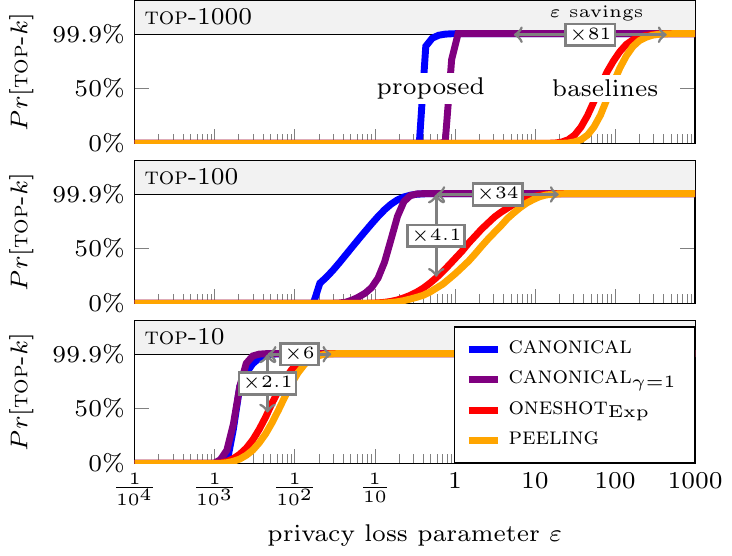}}{\centering\includegraphics[width=0.465\textwidth]{plots/n4top1000.pdf}}
\caption{$\CANONICAL$ and $\CANONICAL_{\gamma=1}$ vastly improve the probability of sampling the true top-$k$ for $\eps$-DP top-$k$.\label{fig:netflix}}
\end{figure}
This ensures that if two objects $\hat{a}$ and $\hat{b}$ differ in the participation of at most one user, then observing $Y = y$ provides little insight into the participation status of a user due to the bounded log-likelihood ratio of participation scenarios $\hat{a}$ vs. $\hat{b}$, i.e., $|\log \frac{Pr[Y = y \mid \hat{a}]}{Pr[Y = y \mid \hat{b}]}| \le \eps$. 
Consider the well-known example of the Netflix prize dataset~\cite{netflix_data_paper} and  queries asking for top-$k$ movies/patterns, based on movie ratings and other user  information. Then $\eps$-DP modifies query answers to limit inferences on user data. 

Our work considers the top-$k$ selection problem under $\eps$-DP, proposing 
novel methods that reduce utility loss both in theory (cf. Section~\ref{sec:theory}) and in practice (cf. Figure~\ref{fig:netflix} for results on the  Netflix prize dataset) and are amenable to simple, efficient implementations. Our main contributions are:
\paragraph{I. Lipschitz Mechanism.} 

Inspired by the report-noisy-max mechanism~\cite{dwork2013algorithmic},  equivalences of the exponential mechanism \cite{mcsherry2007mechanism} and permute-and-flip mechanism \cite{mckenna2020permute} to their additive noise formulations \cite{durfee2019practical,ding2021permute}, we sought a property that unifies them and guarantees $\eps$-DP. We discovered that adding noise $X$ to each score and reporting the index of the maximal noisy score is $\eps$-DP if $\ln(1-Pr[X \le x])$ is $1$-Lipschitz (cf. Theorem~\ref{thm:lipschitz}). Hence, we call this additive noise method the Lipschitz mechanism. This mechanism instantiates many popular mechanisms and novel variants via different choices for noise distributions and parameter $\kappa$: 

\SINGLEDOUBLE{%
\begin{center}
\begin{tabular}{rll}
\toprule
Noise Distribution &  top-$1$ with $\kappa = 1$            & top-$k$ with $\kappa = k$\\
\midrule
Gumbel               & \textsc{Exponential Mechanism} \cite{mcsherry2007mechanism}    & $\PEELING$  \\
Laplace              & \textsc{Report-Noisy-Max} \cite{dwork2013algorithmic}      &  $\ONESHOT_{\text{Laplace}}$ \cite{qiao2021oneshot}\\
(Half-)Logistic         & {\color{gray}(new)}                    & {\color{gray}(new)}\\ 
Exponential          & \textsc{Permute-And-Flip} \cite{mckenna2020permute}       &  $\ONESHOT_{\text{Exp}}$ {\color{gray}(new)} \\
\bottomrule
\end{tabular}
\end{center}%
}{%
{\small%
\begin{tabular}{rll}
\toprule
Noise Distr. &  top-$1$ with $\kappa = 1$            & top-$k$ with $\kappa = k$\\
\midrule
Gumbel               & \textsc{Exp. Mechanism}    & $\PEELING$  \\
Laplace              & \textsc{Report-Noisy-Max}       &  $\ONESHOT_{\text{Laplace}}$ \\
(Half-)Logistic      & {\scriptsize\color{gray}(new)}   & {\scriptsize\color{gray}(new)}\\ 
Exponential          & \textsc{Permute-And-Flip}        &  $\ONESHOT_{\text{Exp}}$ {\scriptsize\color{gray}(new)} \\
\bottomrule
\end{tabular}
}%
}%

Reporting the indices of the $\kappa$ largest utility values is also $\eps$-DP ($\kappa$ reduces the $\eps$ internally) and instantiates the faster oneshot mechanisms, such as the $\ONESHOT_{\text{Laplace}}$ mechanism \cite{qiao2021oneshot}.  The Lipschitz mechanism yields the first $\eps$-DP proof of a oneshot variant of permute-and-flip ($\ONESHOT_{\text{Exp}}$). Due to the versatility of the Lipschitz mechanism, we use it both to instantiate existing work, as well as new top-$k$ methods that apply the mechanism directly on the set of $k$-subsets as the selection domain (with $\kappa = 1$).

\vspace{-3mm}
\paragraph{II. Canonical Lipschitz Mechanism for top-$k$.} 
Applying general selection mechanisms, such as the Lipschitz mechanism, over an exponentially large selection domain is prohibitive in cost. Thus, specialized mechanisms for suitable loss functions are needed. A natural loss function is inspired by the \REVISE{ following well-known trick: a single user can change the required number of users for a candidate solution to become optimal by at most one.}{well-known trick that a single user can change the number of user changes needed (for a solution to become optimal) at most by one.} We investigated how to deploy such a canonical loss function over the domain of $k$-subsets, as recent works indicate strong utility guarantees for canonical functions in general~\cite{asi2020near,asi2020instance,medina2020duff}. Prior work on related loss functions for top-$k$ over $d$ items \cite{joseph2021joint} specialized the exponential mechanism to improve the sampling time from $O(d^k)$ to $\tilde{O}(dk^3)$. As this is still too costly for large $k$, we develop faster methods. By plugging the canonical loss function into the Lipschitz mechanism with $\kappa = 1$, we obtain the Canonical Lipschitz mechanism ($\CANONICAL$) 
and its faster variant ($\CANONICAL_{\gamma = 1}$):

\SINGLEDOUBLE{%
\begin{center}
\begin{tabular}{lll}
\toprule
Mechanism & Time & Space \\
\midrule
$\textsc{naive}$ & $O(d^k)$ & $O(k)$ \\
$\textsc{joint}$ \cite{joseph2021joint} & $O(d k^3 \log k)$ & $O(dk)$ \\
$\PEELING$ \cite{durfee2019practical} & $O(dk)$ & $O(k)$ \\
$\ONESHOT_{\text{Gumbel}}$ \cite{durfee2019practical} & $O(d)$ & $O(k)$ \\
$\ONESHOT_{\text{Laplace}}$ \cite{qiao2021oneshot} & $O(d)$ & $O(k)$ \\
\midrule
$\ONESHOT_{\text{Exp}}$ & $O(d)$ & $O(k)$ \\
$\textsc{canonical}$ & $O(dk+d\log d)$ & $O(k)$ \\
$\textsc{canonical}_{\gamma = 1}$ & $O(d\log d)$ & $O(k)$ \\
\bottomrule
\end{tabular}
\end{center}
}{%
{%
\small%
\begin{tabular}{lll}
\toprule
Runtime & Prior Mechanisms & Contribution \\
\midrule
\footnotesize$O(d^k)$ & naively applied mechanisms & - \\
\footnotesize$\tilde{O}(dk^3)$ & \footnotesize$\JOINT$ \cite{joseph2021joint} & - \\
\footnotesize$\tilde{O}(dk)$ & \footnotesize$\PEELING$ \cite{durfee2019practical} & \footnotesize$\CANONICAL$ \\
\footnotesize$\tilde{O}(d)$ & \footnotesize$\ONESHOT_{\text{Laplace}}$ \cite{qiao2021oneshot} & \footnotesize $\textsc{canon.}_{\gamma = 1}$ \\
\bottomrule
\end{tabular}%
}%
}

\REVISE{$\CANONICAL$ and $\textsc{canonical}_{\gamma = 1}$ are the first approaches that can efficiently instantiate the exponential mechanism to obtain a probability distribution over all subsets . This is not known to be possible with $\PEELING$ or $\ONESHOT$.}{

$\CANONICAL$ can be sampled in $\tilde{O}(dk)$ and $\CANONICAL_{\gamma=1}$ in $\tilde{O}(d)$ time, both storing only a handful auxiliary values (cf., Appendix~\ref{sec:suppimpl}). Their subset probabilities (when instantiating the exponential mechanism) can be obtained in time and space matching their sampling time complexities.
}
\vspace{-0.75cm}
\paragraph{III. Composition-Free vs Sequential Composition.} Traditionally,  top-$k$ has been approached by $k$ repeated $(\eps/k)$-DP selections without-replacement ($\PEELING$), which is $\eps$-DP via sequential composition \cite{mcsherry2009privacy}. 
We show\footnote{We consider two different forms of analysis. One is based on investigating noise terms exploiting the additive noise formulation of the Lipschitz mechanism (cf. Theorem~\ref{lem:leapcomparison}) and the other is based on classical utility loss bounds (cf. Theorem~\ref{thm:utilcomparison}). In both analyses, we obtain the result of an improvement by a factor $\Omega(\log k)$ for $\frac{d}{k} = O(1)$ .} that composition-free methods \REVISE{such as}{like} $\CANONICAL$ improve utility guarantees by a factor $\Omega(\log k)$ compared to $\PEELING$ and our experiments {on synthetic and \REVISE{real}{real-world} data } show similarly conclusive results. 

\SINGLEDOUBLE{\clearpage}{}

Notational conventions can be found below. 

\SINGLEDOUBLE{%
\begin{center}%
\begin{tabular}{l}
\toprule
$\log$ is the natural logarithm, i.e., $\log(\exp(1)) = 1$ \\
$\sgn(0) = 0$ and $\sgn(x) = x/|x|$ for $x \in \mathbb{R}_{\neq 0}$ \\
${d \choose 0} = 1$  and ${d \choose k} = \prod_{i = 1}^k \frac{d+1-i}{i}$ for $d \in \mathbb{Z}_{\ge 0},k \in \mathbb{N} $ \\
$\min(\vec{a}) = \min \{ \vec{a}_1, \ldots, \vec{a}_d \}$ for $\vec{a} \in \mathbb{R}^d$ \\
$\max(\vec{a}) = \max \{ \vec{a}_1, \ldots, \vec{a}_d \}$ for $\vec{a} \in \mathbb{R}^d$ \\
$\|\vec{a}-\vec{b}\|_\infty = \max_{i = 1}^d |\vec{a}_i-\vec{b}_i|$ for $\vec{a},\vec{b} \in \mathbb{R}^d$ \\
$\kargmax{k}{j \in \mathcal{J}} f(j)$ are the $k$ last $j$'s sorted by $f(j)$ \\
(sorting ties are broken arbitrarily) \\
mechanisms select from $\mathbb{Y}$ with privacy loss $\eps \in \mathbb{R}_{\ge 0}$\\
$\tilde{O}(\ldots)$ hides log terms and $\propto$ hides normalization terms \\
\bottomrule
\end{tabular}
\end{center}
}{%
\begin{tabular}{l}
\toprule
$\log$ is the natural logarithm, i.e., $\log(\exp(1)) = 1$ \\
$\sgn(0) = 0$ and $\sgn(x) = x/|x|$ for $x \in \mathbb{R}_{\neq 0}$ \\
${d \choose 0} = 1$  and ${d \choose k} = \prod_{i = 1}^k \frac{d+1-i}{i}$ for $d \in \mathbb{Z}_{\ge 0},k \in \mathbb{N} $ \\
$\min(\vec{a}) = \min \{ \vec{a}_1, \ldots, \vec{a}_d \}$ for $\vec{a} \in \mathbb{R}^d$ \\
$\max(\vec{a}) = \max \{ \vec{a}_1, \ldots, \vec{a}_d \}$ for $\vec{a} \in \mathbb{R}^d$ \\
$\|\vec{a}-\vec{b}\|_\infty = \max_{i = 1}^d |\vec{a}_i-\vec{b}_i|$ for $\vec{a},\vec{b} \in \mathbb{R}^d$ \\
$\kargmax{k}{j \in \mathcal{J}} f(j)$ are the $k$ last $j$'s sorted by $f(j)$ \\
(sorting ties are broken arbitrarily) \\
mechanisms select from $\mathbb{Y}$ with privacy loss $\eps \in \mathbb{R}_{\ge 0}$\\
$\tilde{O}(\ldots)$ hides log terms and $\propto$ hides normalization terms \\
\bottomrule
\end{tabular}%
}

\section{Lipschitz Mechanism for Discrete Selection } \label{sec:lipschitzintro}

\begin{definition}[Lipschitz Mechanism] \label{def:lipschitz}

Let $F, F^{-1}$ be a pair of a cumulative and an inverse distribution function for which $\log(1-F(x))$ is $1$-Lipschitz, i.e., for any $x,c \in \mathbb{R}$:

$$|\log(1-F(x))-\log(1-F(x+c))| \le \REVISE{c}{|c|}$$

 Let $\eps \in \mathbb{R}_{>0}, \kappa \in \mathbb{N}$, $\hat{x} \in \mathbb{X}$ be a  data object, $\mathbb{Y}$ be the selection domain of the mechanism,  and $U_y \sim Unif(0,1)$ be independent for each possible output $\y \in \mathbb{Y}$. Let ${\loss( y \mid \hat{x} )}$ for $y \in \mathbb{Y}$ have sensitivity $\Delta_{\textsc{loss}}$. Then the output of the Lipschitz mechanism is:
$$\{Y_1,\ldots, Y_\kappa\} = \kargmax{\kappa}{\y \in \mathbb{Y}} \left\{ \frac{ \, \loss(y \mid \hat{x})}{-2\kappa\Delta_{\textsc{loss}} / \eps}  +F^{-1}( U_\y ) \right\}$$
\end{definition}

Each $\loss$ value is negated/rescaled and then distorted by adding noise terms generated via inverse transform sampling. Then the items with the $\kappa$ largest noisy values are reported. The $\eps$-DP proof (cf. Theorem~\ref{thm:lipschitz}) follows from the Lipschitz condition. This condition is \REVISE{satisfied e.g.,}{for instance satisfied} by the standard  Laplace, Gumbel, Exponential Distribution and  {(Half-)Logistic} distribution functions (cf. Appendix~\ref{sec:supplipschitz}). 

The standard exponential distribution and $F^{-1}(p) = -\ln(1-p)$ satisfies the Lipschitz condition tightly \REVISE{($1-F(x)) = \exp(-x)$)}{($\log(1-F(x)) = -x$)}, has the smallest variance amongst the considered distributions, and consistently performed best in our experiments (cf. Figure~\ref{fig:lipschitz} in Appendix~\ref{sec:repl}).  The Laplace distribution matches the exponential distribution except that it \REVISE{multiplies generated values with a random sign, i.e.,their absolute value is}{selects a random sign i.e., absolute values are still} exponentially distributed. \REVISE{It}{Laplace noise} tends to \REVISE{adds more noise}{distort the loss values more than exponential noise} due to opposing signs. The Gumbel distribution with $F^{-1}(p) = -\ln(- \ln(p))$ is often used as the default to simplify the analysis and because it allows to easily derive probabilities as it instantiates the exponential mechanism and probabilities are therefore proportional to simple exponential terms. \REVISE{}{These benefits come at the price of slightly higher loss value distortion.}

\subsection{Lipschitz mechanism for Discrete Selection over exponentially large selection domains}  \label{sec:fastlipschitz}

Naively sampling the Lipschitz mechanism over a selection domain $\mathbb{Y}$ takes $O(\kappa |\mathbb{Y}|)$ time. This is problematic if $\mathbb{Y}$ is exponentially large as is the domain of  top-$k$ selection. Yet, it can be evaded if $\loss$-values can be grouped into few classes with distinct $\loss$-values. For large groups with the same loss value, only the group members with the $\kappa$ largest noise values can ever be selected and each group member has the same probability to receive those noise values. Hence, it suffices to generate the $\kappa$ largest noise values per group. This can be done efficiently using order statistics for uniform random variables, because $F^{-1}$ is strictly increasing and each needed order statistic can be generated sequentially in $O(1)$ arithmetic operations \cite{lurie1972machine}. In the special case of $\kappa = 1$, only one noise term per group needs to be generated. Let $U_0, \ldots, U_m$ be i.i.d. standard uniform random variables used to generate the noise terms of a group of size $m$ via inverse transform sampling. From order statistics of standard uniform variables, it follows that ${U_0}^{1/m}$ is distributed equally to $\max \{ U_1, \ldots, U_m \}$. Thus, $F^{-1}( {U_0}^{1/m})$ directly generates the maximal noise term of the group and its elements have a uniform probability $\frac{1}{m}$ of receiving this maximal noise term. 

\section{Canonical Lipschitz Mechanism for Top-$k$} \label{sec:canonicallipschitz}

\subsection{Canonical Loss Function}

Let $f$ be some scoring function over a set of items $\{1, \ldots, d\}$ with sensitivity $\Delta_f$. Let $\mathbb{Y}$ be all $k$-subsets of $\{1, \ldots, d\}$. Let for any $\y \in \mathbb{Y}$ the utility loss $\loss(y \mid \vec{x}) $ be a shorthand for $\loss(y \mid \hat{x})$, where each component $\vec{x}_j = f( j \mid \hat{x} )/\Delta_f$ for $j \in \{1, \ldots, d\}$.

 A common loss function for many problems is to quantify how much the data object $\hat{x}$ would need to change in terms of participants  $\textsc{users}(\hat{x})$ for a potential output $y$ to become optimal.  This has been shown to offer strong utility guarantees for a large class of problems \cite{asi2020near,asi2020instance,medina2020duff}.  The optimality zone $OPT^{-1}(y)$ of a potential output $y$ is the vector space of all $d$-dimensional vectors whose $k$ largest values are in the indices  $y_1,\ldots,y_k$.
 The canonical loss function is then the distance to the nearest data vector in $OPT^{-1}(y)$:

$$\loss(y \mid \vec{x}) =  \min_{\, \vec{v} \in \OPT^{-1}(y) \,} \|\vec{x}-\vec{v}\|_\infty$$

which has sensitivity $1$ (see Lemma~\ref{thm:canonicalloss} in Appendix~\ref{sec:proofs}).

This definition of canonical loss is real-valued unlike in \cite{asi2020near,asi2020instance,medina2020duff}. However, after rounding up the value can be interpreted as the number of users needed to transform $\vec{x}$ into $\vec{z}$ s.t. $y \in \OPT(\vec{z})$:

\begin{lemma}
The function $\loss^{*}(y \mid \vec{x}) = \lceil \loss(y \mid \vec{x}) \rceil$ has sensitivity $\Delta_{\textsc{loss}^{*}} = 1$.
\end{lemma}
\begin{proof}
The canonical loss function $\loss(y \mid \vec{x}) \ge 0$ and has sensitivity $1$. Thus, if $a \in \mathbb{R}_{\ge 0}$ is replaced with a value in $[\max(0,a-1),a+1]$, then $b = \lceil a \rceil$ is replaced by an integer between $b-1 = \lceil a-1 \rceil $ and $b+1 = \lceil a+1 \rceil$.
\end{proof}

\subsection{Top-$k$}

\begin{definition}[score vector]
Let $f$ be a scoring function with sensitivity $\Delta_f$. 
Items $\{1, \ldots, d\}$ are assigned 
scores $\vec{x} \in \mathbb{R}^d$ with $\vec{x}_j = f(j \mid \hat{x}) / \Delta_f$ for $j \in \{1, \ldots, d\}$. The (descending) order statistics of the components of $\vec{x}$  are:
$$\top{1} \ge \top{2} \ge \ldots \ge \top{d}$$

Let $\topitem{1}, \ldots, \topitem{d} \in \{1, \ldots, d\}$ be indices such that:
$$\vec{x}_{\topitem{1}} = \top{1}, \, \, \vec{x}_{\topitem{2}} = \top{2}, \, \, \ldots, \, \, \vec{x}_{\topitem{d}} = \top{d}$$

\end{definition}

The top-$k$ of a score vector are its $k$-largest components:

\begin{definition}[top-$k$] Let $\vec{x} \in \mathbb{R}^d$ and $k\in \mathbb{N}$. Then:
$$\OPT(\vec{x}) = \{\topitem{1}, \ldots, \topitem{k} \}= \kargmax{k}{i \in \{1, \ldots, d\}} \{ \vec{x}_i \}$$
\end{definition}

As releasing the top-$k$ under DP  is not always possible, it is useful to identify groups of subsets that are good approximations. For this purpose, all $k$-subsets are partitioned into disjoint utility classes based on an integer $h \in \{1, \ldots, k-1\}$ and an integer $t \in \{k, \ldots, d\}$, where $t = k$ is only allowed if $h = k-1$.  The integer $h$ \REVISE{is}{relates to} the highest missing rank from the subset and $t$ is the lowest present rank in the subset:

\begin{definition}[utility class]\label{def:subsetclass}
The utility class $\mathcal{C}_{h,t}$ is comprised of all subsets of the form $\{\topitem{1}, \ldots, \topitem{h} \} \, \cup \mathcal{B} \, \cup \{ \topitem{t} \}$ with $\mathcal{B} \subseteq \{ \topitem{h+1}, \ldots, \topitem{t-1} \}$ and $h+|\mathcal{B}|+1 = k$.
\end{definition}

Thus, each subset $y \in \mathcal{C}_{h,t}$ fully contains the top-$h$, while the remaining $k-h$ items are contained in the top-$t$. The subset $\{\topitem{1}, \ldots, \topitem{h} \}$ \REVISE{is referred to as}{can be thought of as the} ``head'', $\mathcal{B}$ as the ``body'' and $\topitem{t}$ as the ``tail'' of a subset. 

Let $y \in \mathcal{C}_{h,t}$ be a $k$-subset of $\{1, \ldots, d\}$, then (cf.  Lemma~\ref{lem:topkcanonical} in Appendix~\ref{sec:proofs}):

$$\loss(y  \mid \vec{x}) = \min_{\, \vec{v} \in \OPT^{-1}(y) \,} {\|\vec{x}-\vec{v}\|_\infty} = \frac{\top{h+1}-\top{t}}{2}$$

Intuitively, $\top{h+1}$ is the best item missing from $y$ and even the worst item $\top{t}$ in $y$ needs to overtake $\top{h+1}$. The factor $\frac{1}{2}$ is due to the fact that users can both increase $\top{t}$ by $1$ and decrease $\top{h+1}$ by $1$ such that overtaking takes half as much effort. This can be generalized by a parameter $\gamma \in [0,1]$ (matching the previous definition for $\gamma = \frac{1}{2}$ ):

\begin{theorem}[top-$k$ canonical loss function] \label{thm:topkcanonfunction}

Let $y \in \mathcal{C}_{h,t}$.

$$\loss(y  \mid \vec{x}) = (1-\gamma)\,\top{h+1}-\gamma\,\top{t}$$

has sensitivity $1$ for any $\gamma \in [0,1]$.
\end{theorem}
\begin{proof}
If all values of a set change at most by $C \in \mathbb{R}$, then their extrema can also change at most by $C$ (see Lemma~\ref{lem:extrema} in Appendix~\ref{sec:proofs}). As the values of $\vec{x}_j$ for $j \in \{1, \ldots, d\} \setminus y$ and the values for $\vec{x}_j$ for $j \in y$ change at most by $1$, the values of $\top{h+1}$ and $\top{t}$ also change at most by  $1$, since $\top{h+1} = \max_{j \in \{1, \ldots, d\} \setminus y} \vec{x}_j$ and $\top{t} = \min_{j \in y} \vec{x}_j$.

The term $(1-\gamma)\,\top{h+1}$ can therefore change at most by $(1-\gamma)$ and the term $\gamma\,\top{t}$ can change by at most  $\gamma$. Thus, the difference can change by at most $(1-\gamma)+\gamma = 1$. 
\end{proof}

By plugging the function in Theorem~\ref{thm:topkcanonfunction} into the Lipschitz mechanism with $\kappa = 1$ and using the techniques in Section~\ref{sec:fastlipschitz} to deal with exponentially large selection domains, we obtain the Canonical Lipschitz mechanism, which can be sampled in $O(dk)$ (cf. Theorem~\ref{thm:implementation} in Appendix~\ref{sec:suppimpl}). 

Numerical precision can be either achieved by taking computations into the log-space\footnote{
One can use $F^{-1}(p) = -\log(-\log(p))$ from the Gumbel distribution where it simplifies to $F^{-1}( {U_{h,t}}^{1/|\mathcal{C}_{h,t}|}) = \ln(|\mathcal{C}_{h,t}|) + F^{-1}({U_{h,t}})$. As $|\mathcal{C}_{h,t}|$ are binomial coefficients that can be computed via multiplications, it is trivial to take the computations into the log-space.} or  via special libraries.

While Theorem~\ref{thm:implementation} applies to any $\gamma\in[0,1]$, the mechanism can be sampled in $O(d)$ time for $\gamma=1$, as only each $t \in \{k, \ldots, d\}$ has a distinct loss value. In this case \REVISE{}{there are} $\sum_{h = 0}^{k-1} |\mathcal{C}_{h,t}| = {{t-1} \choose {k-1}}$ \REVISE{}{subsets for each loss value $\top{t}$}. 

\newcommand{\G}[1]{\mathcal{G}^{\Data{}}_{#1}}

\newcommand{\given}{\, | \,}

\section{Differential Privacy Guarantees} \label{sec:lipschitz}

All approaches considered in this work are instantiated through the Lipschitz mechanism, which only requires a distribution choice s.t. $\ln(1-F(x))$ is $1$-Lipschitz for any $x \in \mathbb{R}$ and the $\loss$ function to have some known sensitivity $\REVISE{\Delta}{\Delta_{\loss}}$. Thus, it suffices to prove:

\begin{theorem} \label{thm:lipschitz}
 
The Lipschitz mechanism from Definition~\ref{def:lipschitz} is $\eps$-DP \REVISE{}{for any $\loss$ function with finite sensitivity $\Delta_{\loss}$}. 

\end{theorem}
\begin{proof}

We use here the same variables as in Definition~\ref{def:lipschitz}. Let $\Delta$ be a shorthand for $\Delta_{\textsc{loss}}$. Let $a = a_1, \ldots, a_\kappa$ be an arbitrary $\kappa$-subset of $\mathbb{Y}$ and $b = \{ b_1, \ldots, b_{d-\kappa} \} = \mathbb{Y} \setminus a$ be 
the $d-k$ items missing from $a$. Let the loss $\vec{a} = [{\loss(a_1 \mid \hat{x})}, \ldots, {\loss(a_\kappa  \mid \hat{x})} ]^T$ and $\vec{b} = [{\loss(b_1  \mid \hat{x})}, \ldots, {\loss(b_{d-\kappa} \mid \hat{x})}]^T$. Let the noisy normalized loss $\vec{A} = \frac{\eps}{2 \kappa \Delta} \vec{a} + [F^{-1}(U_1), \ldots, F^{-1}(U_\kappa) ]^T$ and $\vec{B} = \frac{\eps}{2 \kappa \Delta} \vec{b} + [F^{-1}(U_{\kappa+1}), \ldots, F^{-1}(U_d) ]^T$. 

Let $Y = \{Y_1,...,Y_k\}$. Then based on minimal and maximal component values of $\vec{A}$ and $\vec{B}$ the probability $Pr[Y = a]$ is equal to:

$$\int_{-\infty}^{+\infty}  Pr[\min(\vec{A}) > x] \cdot Pr[\max(\vec{B}) = x] \, dx$$

The integration variable $x$ is the $(\kappa+1)$-largest noisy value overall. Thus, it covers all possible events where $\vec{A}$ are the $\kappa$-largest noisy values and all events are disjoint, differing at least in the $(\kappa+1)$-largest noisy value.

Let $\omega \in [-1,+1]$. Then replacing $\vec{a}$ with $\vec{a}+2\omega\Delta$ maximizes a user's impact on $Pr[Y = a]$, due to the following properties that follow from the additive noise framework:

\begin{itemize}
\item[(i.)] Monotonicity: Replacing $\vec{a}$ with $\vec{a}+\omega\Delta$ and  $\vec{b}$ with 
$\vec{b}-\omega\Delta$ maximizes a user's impact on $Pr[Y = a]$.
\item[(ii.)] Shift-Invariance: Replacing both $\vec{a}$ and $\vec{b}$ with $\vec{a}+\omega\Delta$ and $\vec{b}+\omega\Delta$ has no effect on $Pr[Y = a]$.
\end{itemize}

It is easy to see that (ii.) can be used to cancel out changes to $\vec{b}$ in (i.) by doubling down on changes to $\vec{a}$. For $\kappa = 1$, this corresponds to the claims about regular mechanisms in \cite{mckenna2020permute}.

As $Pr[\max(\vec{B}) = x]$ is independent of $\vec{a}$, it is only left to show that $Pr[\min(\vec{A}) > x]$ changes by at most a multiplicative factor $\exp(\eps)$ due to a single user.

From inverse transform sampling with $U_i \sim Unif(0,1)$ follows $Pr[ \vec{A}_i > x] = Pr[ \left(\frac{\eps}{2 \kappa \Delta} (\vec{a}_i) + F^{-1}(U_i)\right) > x]$ is for any $i \in \{1, \ldots, \kappa\}$ equal to:
\begin{align*}
 Pr\Big[ \, F^{-1}(U_i)  > \,  x-\frac{\eps}{2\kappa\Delta} \vec{a}_i \, \Big] = \, 1 \, - F\big(x-\frac{\eps}{2\kappa\Delta} \vec{a}_i\big) \\ 
\end{align*}


Replacing the normalized loss $\vec{a}$ with $\vec{a}+2\omega\Delta$ replaces $\frac{\eps}{2\kappa\Delta}\vec{a}$ with $\frac{\eps}{2\kappa\Delta}(\vec{a}+2\omega\Delta) = \frac{\eps}{2\kappa\Delta} \vec{a}+\omega \frac{\eps}{\kappa}$. This means the noisy normalized loss $\vec{A}$ is replaced with $\vec{A}+\omega \frac{\eps}{\kappa}$ and the probability $Pr[\min(\vec{A}+\omega \frac{\eps}{\kappa}) > x]$ is equal to:

$$\prod_{i = 1}^\kappa Pr[ \vec{A}_i+\omega \frac{\eps}{\kappa} > x] = \prod_{i = 1}^\kappa \left[ 1-F(x-\frac{\eps}{2\kappa\Delta}\vec{a}_i -\omega \frac{\eps}{\kappa}) \right]$$

The $1$-Lipschitz condition on $\log(1-F(x))$ implies $\frac{1-F(x)}{\exp(c)} \le 1-F(x+c) \le \exp(c)(1-F(x))$ for any $c \in \mathbb{R}_{\ge 0}$. Thus, with $c = \omega \frac{\eps}{\kappa}$ the value of $1-F(x-\frac{\eps}{2\kappa\Delta}\vec{a}_i -c)$ fluctuates at most by a multiplicative factor $\exp(c) \le \exp(\eps/\kappa)$ and $Pr[\min(\vec{A} > x]$ fluctuates at most by a multiplicative factor $\prod_{i=1}^\kappa \exp(\eps/\kappa) = \exp(\eps)$ due to $\omega \in [-1, +1]$.
\end{proof}

\section{Theoretical Comparison} \label{sec:theory}

\subsection{Comparison of Canonical for different $\gamma$}

The subset classes $\mathcal{C}_{h,t}$ have better utility for larger $h$ and smaller $t$. The parameter $\gamma$ governs how much $\CANONICAL$ prioritizes $t$  against $h$. Inversely, for small values of $\gamma$ the  mechanism improves $h$, but neglects $t$. By trading privacy loss for utility \REVISE{it}{\CANONICAL{} with smaller $\gamma$} can \REVISE{improve for $t$ as well and}{become as good as } \REVISE{catch up with }{\CANONICAL{} with} larger $\gamma$  \REVISE{values}{at minimizing $t$}:

\begin{lemma}[internal superiority] \label{lem:superior}
For $\Gamma_2 > \Gamma_1$ and $\eps_1 = ({\Gamma_2}/{\Gamma_1}) \eps_2 < \eps_2$,  $\eps_1$-DP $\CANONICAL_{\gamma = \Gamma_1}$ has superior utility to $\eps_2$-DP $\CANONICAL_{\gamma = \Gamma_2}$, i.e., for any $y,y' \in \mathbb{Y}$ 

{
\footnotesize
\begin{align*}
\frac{\left|Pr[y \mid \hat{x},\gamma = \Gamma_1, \eps = \eps_1]-Pr[y' \mid \hat{x},\gamma = \Gamma_1, \eps = \eps_1]\right|}{\left|Pr[y \mid \hat{x},\gamma = \Gamma_2, \eps = \eps_2]-Pr[y' \mid \hat{x},\gamma = \Gamma_2, \eps = \eps_2]\right|} > 1 \\
\end{align*}
}
\end{lemma}
\begin{proof}

Let $y \in \mathcal{C}_{h,t}$ and $y' \in  \mathcal{C}_{h',t'}$ (cf. Definition~\ref{def:subsetclass}). 
\REVISE{The following hold}{Then the following holds}: 
\begin{align*}
\frac{|\eps_1 \, \loss(y \mid \hat{x}, \gamma = \Gamma_1)-\eps_1 \, \loss(y' \mid \hat{x}, \gamma = \Gamma_1)|}{|\eps_2 \, \loss(y \mid \hat{x}, \gamma = \Gamma_2)-\eps_2 \,  \loss(y' \mid \hat{x}, \gamma = \Gamma_2)|} & > 1 \\
\frac{\eps_1| (1-\Gamma_1)(\top{h+1}-\top{h'+1})-\Gamma_1(\top{t}-\top{t'}) |}{\eps_2| (1-\Gamma_2)(\top{h+1}-\top{h'+1})-\Gamma_2(\top{t}-\top{t'}) |} & > 1 \\
\end{align*}

Since $\Gamma_2 > \Gamma_1$, \REVISE{}{ then }$(1-\Gamma_1) > (1-\Gamma_2)$ \REVISE{and thus}{ and } $\frac{\eps_1}{\eps_2}$ just needs to be large enough to scale $\Gamma_1$ to $\Gamma_2$. This is achieved due to $\frac{\eps_1}{\eps_2} = {\Gamma_2}/{\Gamma_1}$, which is implied by $\eps_1 = \REVISE{}{(}{\Gamma_2}/{\Gamma_1} \REVISE{}{)} \eps_2$.
\end{proof}

\begin{corollary}\label{cor:superior}
$\frac{\eps}{\gamma}$-DP $\CANONICAL_{\gamma}$ has superior utility to $\eps$-DP $\CANONICAL_{\gamma = 1}$ as shown in Lemma~\ref{lem:superior}.
\end{corollary}

\subsection{Noise Analysis of Canonical vs Peeling }

Let $\eps = 2\Delta = 1$ without loss of generality. For the Lipschitz mechanism with Gumbel noise, \REVISE{we can exploit that}{} the maximum $G_{(n)}$ of $n$ i.i.d. standard Gumbel \REVISE{RVs}{random variables (RVs)} $G_1, \ldots, G_n$ is distributed $G_{(n)} \sim \log(n)+Gumbel(0,1)$ and the difference of two standard Gumbel RVs follows a standard Logistic RV. We seek a a specific difference that instructs how far non-top-$k$ options leap forward compared to top-$k$ ones, hence called Logistic Leap.

To obtain simpler formulas we add for $\CANONICAL$ an additional subset with the same loss value as the previously worst one to the selection domain $\mathbb{Y}$, which can only disadvantage $\CANONICAL$.

\begin{lemma}[Canonical Logistic Leap] \label{lem:canonicalleap}
Let $\mathbb{Y}$ be the selection domain with an additional (dummy) subset s.t. $|\mathbb{Y}| = {d \choose k}+1$. 
Let $y\in \mathbb{Y}$ be the subset selected by $\CANONICAL$ with parameter $\gamma \in (0, 1]$. \REVISE{It holds that}{Let}: 
\begin{align*}
X \sim  \frac{1}{\gamma} \left( {k \log(d/k)}+{\log(c_{d,k})}+\text{Logistic}(0,1) \right) \\
\end{align*}
with $0 < \ln(c_{d,k}) = \ln\left({ {d \choose k} }/{\frac{d^k}{k^k} }\right) \le k$.  Then:
\begin{align*}
Pr\left[\,\,|\OPT \cap y| > 0\,\right] &\ge Pr[X > \top{1}-\top{d}] \\
Pr\left[\,\,|\OPT \cap y| < k\,\right] &\le Pr[X < \top{k}-\top{k+1}] \\
\end{align*}
\end{lemma}
\begin{proof}
\newcommand{\NN}[1]{N_{1,#1}}

 Let $\mathbf{N}$ be a $1 \times |\mathbb{Y}| $ matrix whose entries are independent noise terms:
$\mathbf{N} = \begin{psmallmatrix}
\REVISE{\color{gray} \NN{1}}{\NN{1}} & \NN{2} & \cdots & \NN{|\mathbb{Y}|} \\
\end{psmallmatrix}.
$

Without loss of generality we can fix $\gamma = 1$ and then replace $\eps$ by $\frac{\eps}{\gamma}$ as according to Corollary~\ref{cor:superior} this incurs no utility loss. Replacing $\eps$ with $\frac{\eps}{\gamma}$ is equivalent to multiplying all noise terms by $\frac{1}{\gamma}$.

 $\NN{1}$ is the noise-term received by the top-$k$, while the sub-matrix $\mathbf{S}$ of $\mathbf{N}$ with $\NN{j}$ where  $2 \le j \le |\mathbb{Y}|$ are the noise-terms received by the non-top-$k$ subsets. 
 
 Let $\NN{m} = \max(\mathbf{S})$. For $\CANONICAL_{\gamma\REVISE{\gamma_1}{=}1}$, we can define $X = \NN{m}-\NN{1}$. The non-top-$k$ subset that receives the noise term $\NN{m}$ leaps ahead of the top-$k$ subset if $X > \top{1}-\top{d}$, but fails to overtake the top-$k$ subset if $X < \top{k}-\top{k+1}$. Due to these implications, the probability inequalities in the statement hold.
 
 With $n = |\mathbb{Y}|-1 = {d \choose k}$ non-top-$k$ subset noise terms we get $X = \log( {d \choose k})+Logistic(0,1)$ and we can then write $\ln {d \choose k}$ as $k \ln(d/k)+ \ln(c_{d,k})$ with $0 < \ln(c_{d,k}) \le k$.
\end{proof}

 As $\PEELING$ removes in each round an item and we want to simplify it, we charitably remove the item with utility $\top{d}$, which results in the elimination of more noise terms assigned to non-top-$k$ items (helping $\PEELING$) as would occur due to the item selection:

\begin{lemma}[Peeling Logistic Leap] \label{lem:peelingleap}
Let $y$ be subset selected by $\PEELING$  from items $\{1, \ldots, d\} \setminus \{\topitem{d}\}$. \REVISE{It holds that}{Let}: 
\begin{align*}
X' \sim {k(\log((d-1-k)k)}+&k \cdot \text{Logistic}(0,1) \\
\end{align*}
Then:
\begin{align*}
Pr\left[\,\,|\OPT \cap y'| > 0\,\right] &\ge Pr[X' > \top{1}-\top{d}] \\
Pr\left[\,\,|\OPT \cap y'| < k\,\right] &\le Pr[X' < \top{k}-\top{k+1}] \\
\end{align*}
\end{lemma}
\begin{proof}

Let $\mathbf{N}'$ be a $k \times d$ matrix whose entries are independent noise terms:

\newcommand{\NNN}[2]{N_{#2,#1}'}
\begin{align*}
\mathbf{N'} = \begin{psmallmatrix}
\NNN{1}{1}  & \NNN{2}{1}    & \cdots        & \NNN{k}{1}    & \NNN{k+1}{1}  & \cdots & \NNN{d}{1} \\
\NNN{1}{2}  & \NNN{2}{2}    & \cdots        & \NNN{k}{2}    & \NNN{k+1}{2}  & \cdots & \NNN{d}{2} \\
\vdots      & \ddots        & \vdots        & \vdots        & \vdots        & \ddots & \vdots \\
\NNN{1}{k}  & \NNN{2}{k}    & \cdots        & \NNN{k}{k}    & \NNN{k+1}{k}    & \cdots & \NNN{d}{k} \\
\end{psmallmatrix}
\end{align*}

For $\PEELING$, $N_{i,j}'$ is the noise term received by the item with score $\top{i}$ in round $j$, while  $\mathbf{S}'$ is the submatrix of $\mathbf{N}'$ (depicted as right half)  with $\NNN{j}{i}$ where $1 \le i \le k$ and $k+1 \le j \le d$  with the bottom noise terms, i.e., noise terms of items in the bottom partition ``after'' the top-k. Let $\NNN{m}{r} = \max(\mathbf{S}')$. For $\PEELING$, we can define $X' = \NNN{m}{r}-\NNN{j}{r}$. The bottom item $m$ that receives the noise term $\NNN{m}{r}$ leaps ahead of top item $j$ in round $r$ if $X' > \top{j}-\top{m}$, but fails to overtake $j$ if $X' < \top{j}-\top{m}$. Similarly, if $X' > \top{1}-\top{d} \ge \top{j}-\top{m}$ then $j$ is displaced and if $X' < \top{k}-\top{k+1} \le \top{j}-\top{m}$ the top-item $j$ is not displaced by any bottom item. The probability inequalities in the statement follow from these implications.

With $n' = (d-1-k)k$ bottom noise terms we get $X' = k\left( \log(\, (d-1-k)k \, ) + Logistic(0,1)\right)$.
\end{proof}

\begin{theorem} \label{lem:leapcomparison}
Let $X, X'$ be logistic leap \REVISE{RV}{RVs} from Lemma~\ref{lem:canonicalleap} and Lemma~\ref{lem:peelingleap}.
Then for $\frac{d}{k} = O(1)$ it holds: $$\frac{E[X']}{E[X]} =\Omega(\log k)$$
\end{theorem}
\begin{proof}
Since $\gamma\in (0,1]$, $\frac{1}{\gamma} = O(1)$. Then, 
for $\frac{d}{k} = O(1)$, we rewrite $X$ via $L \sim Logistic(0,1)$ as $X = O(k+L)$ and $X' = O(k \log(k) +kL)$. As $E[L] = 0$, the claim then follows due to $\log k = \Omega(\log k)$ and $k = \Omega(\log k)$.
\end{proof}

\subsection{Utility Loss Bounds for Canonical vs Peeling}

Based on standard utility guarantees for the exponential mechanism one can derive:

\begin{theorem} \label{thm:utilcomparison} 
Let $Y_1, \ldots, Y_k$ be the selected set by $\CANONICAL$ with $\gamma \in (0,1]$ supposing $\frac{1}{\gamma} = O(1)$. Let $Y_1', \ldots, Y_k'$ be the outputted set by $\PEELING$.
Let $T = \argmin_{i \in \{Y_1, \ldots, Y_k \}} \top{i}$ and $T' = \argmin_{i \in \{Y_1', \ldots, Y_k' \}} \top{i}$. 

Let $\alpha \in (0, 0.1]$ be the failure rate. Then with at least probability $1-\alpha$ it holds that $ \top{T} < \top{k}+\frac{2\Delta }{\gamma \eps} \mathcal{E} $ and $ \top{T'} < \top{k}+\frac{2\Delta }{\eps} \mathcal{E}'$
with utility loss terms:
\begin{align*}
\mathcal{E} = \left( k \ln(d/k)+\ln \frac{1}{\alpha}+ k \right) \\
\mathcal{E}' = \left( k \ln(d k) +k \ln \frac{1}{\alpha}- \frac{6}{100} k \right)  \\
\end{align*}

Also, for $\frac{d}{k} = O(1)$, it holds that  ${\mathcal{E}'}/{\mathcal{E}} = \Omega(\log k )$.
\end{theorem}
\begin{proof}

The proof follows from several Theorems and Lemmas in Appendix~\ref{sec:proofs}. For $\CANONICAL$:

\begin{itemize}
\item Lemma~\ref{lem:emguarantees} adopts standard theorems for the exponential mechanism to obtain general utility guarantees (instantiated by Lipschitz mechanism with $F^{-1}$ from the Gumbel distribution).
\item Lemma~\ref{lem:canonutil} plugs the loss function with $\gamma = 1$ over the ${d \choose k}$ subsets from Theorem~\ref{thm:topkcanonfunction}  into Lemma~\ref{lem:emguarantees} to obtain the inequality for $\top{T}$ as in the claim. It is then generalized to $\gamma \in (0, 1]$ via Corollary~\ref{cor:superior}.
\end{itemize}

For $\PEELING$, Lemma~\ref{lem:peelutil} plugs the score function over the $d$ items into Lemma~\ref{lem:emguarantees}, but uses a reduced failure rate $\alpha'$ s.t. $k$ selections \REVISE{still have}{have} a joint \REVISE{failure rate ${(\alpha')}^k \le \alpha$}{success rate ${(1-\alpha)}^k \ge 1-\alpha$}. 

Due to independence one can here use the \Sidak{} correction $\alpha' = 1-(1-\alpha)^{1/k}$ for family wise error rates of hypothesis tests . As this leads to terms that complicate comparisons, $\alpha'$ is rewritten via the \Bonferroni{} correction $\alpha/k$ and a ratio between both corrections $r_{\alpha,k}$ is used to restore the \Sidak{} correction. We derive in Lemma~\ref{lem:bonferroni} that the \Bonferroni{} correction is as expected a very good approximation and that $r_{\alpha,k} \le \REVISE{(}{}{\ln(\frac{1}{1-\alpha})}/{\alpha}$ which even for $\alpha < 0.1$ is smaller than $1.06$ such that $\log(r_{\alpha,k}) < \log(1.06) < \frac{6}{100}$. From that then follows the inequality for $T'$ in the claim. For $\frac{d}{k} = O(1)$, $\mathcal{E}=O\left( k + \ln \frac{1}{\alpha} \right)$ and $\mathcal{E}'=O\left( k \ln(k) +k \ln \frac{1}{\alpha} \right)$. Thus, $\frac{\mathcal{E'}}{\mathcal{E}}=\Omega(\log k)$.
\end{proof}

\section{Related Work}\label{sec:related}

\newcommand{\exponentialmechanism}{EM}
\newcommand{\permuteandflip}{P\&F}

The report-noisy-max mechanism \cite{dwork2013algorithmic} adds Laplace noise to utility values and then selects the item with the \REVISE{max}{maximal} noisy value. Other popular mechanisms can be formulated in a similar way, i.e., 
by adding instead Gumbel noise \cite{durfee2019practical} one gets the exponential mechanism (EM{})  
and by adding Exponential noise \cite{ding2021permute} one gets the permute-and-flip (\permuteandflip{}) mechanism. In this work, we extend these results and unification efforts via the proposed Lipschitz mechanism. We model it as a single mechanism rather than a framework or family of mechanisms as the DP proof is independent of instantiations. We show that in the context of the Lipschitz mechanism asymmetric sensitivities~ \cite{dong2020optimal} can be reduced to ordinary sensitivities, generalizing  results on \REVISE{count-based}{monotonic} functions \cite{dwork2013algorithmic,mckenna2020permute} that are treated to have sensitivity $\frac{1}{2}$. Due to its generality, the Lipschitz  mechanism also instantiates oneshot variants that select the $\kappa$ largest noisy values \cite{qiao2021oneshot}, e.g., the oneshot variant of \permuteandflip{} \cite{mckenna2020permute} did not have a DP  proof although it promises the best utility amongst oneshot mechanisms.

The joint \exponentialmechanism{} \cite{joseph2021joint} is a mechanism that directly selects $k$-subsets based on a loss function akin to canonical loss.   Aside from the problem definition, the joint  \exponentialmechanism{} paper employs a loss function that counts how many users are needed to change the utility values of the subset to match the utility values of the top-$k$. Catching up with the top-$k$ may not displace all items of the top-$k$ and instead yields a mix of top-$k$ and subset items. In contrast, the canonical loss function requires all missing top-$k$ items to be displaced by subset items, which is desirable as it requires more users (cf. Lemma~\ref{lemma:jointcanonical} in Appendix~\ref{sec:proofs}). The joint \exponentialmechanism{} can be sampled in $\tilde{O}(dk^3)$ time, whereas all proposed methods with the canonical loss function require only $\tilde{O}(dk)$ time. The authors of \cite{joseph2021joint} mention as a caveat of joint \exponentialmechanism{} that it may be difficult to avoid exponentially large values in the matrix multiplications that are needed to compute loss value multiplicities. In contrast, the canonical loss value multiplicities ($|\mathcal{C}_{h,t}|$ from Definition~\ref{def:subsetclass}) are binomial coefficients which can be computed with simple methods in the log-space to avoid large values.

Previous works did not theoretically compare direct subset selection with composition methods, but have shown that canonical loss functions offer general utility guarantees \cite{asi2020near,asi2020instance,medina2020duff} if they are plugged into the \exponentialmechanism{} \cite{asi2020instance}. Our motivation to generalize beyond the \exponentialmechanism{} are results for \permuteandflip{} which show it  consistently improves  upon the \exponentialmechanism{} \cite{mckenna2020permute}. Our experiments also indicate for the Lipschitz mechanism  that using $F^{-1}$ from the Exponential distribution leads to best utility (cf. Figure~\ref{fig:lipschitz} in Appendix~\ref{sec:repl}). Last, there are several works \cite{chaudhuri2014large,uai_21,cesar2021bounding} that focus on approximate DP \cite{dwork2014algorithmic,nissim}, i.e, $(\eps,\delta)$-DP with $\delta > 0$.  This work focuses on pure $\eps$-DP, leaving the $\delta > 0$ related question open.

\section{Empirical Comparison}

We compared four top-$k$ mechanisms:

\begin{itemize}
\item $\PEELING$ \cite{durfee2019practical} is one-by-one $\frac{\eps}{k}$-DP selection without replacement via Exp. Mechanism 
\item $\ONESHOT_{\text{Exp}}$ is a oneshot variant \cite{qiao2021oneshot} of permute-and-flip \cite{mckenna2020permute}.
\item $\CANONICAL$ from Section~\ref{sec:canonicallipschitz} (cf. Appendix~\ref{sec:suppimpl} for details) with $\gamma = \frac{1}{2}$ and $F^{-1}$ from Gumbel distribution.
\item $\CANONICAL_{\gamma = 1}$ as $\CANONICAL$ except with $\gamma = 1$.
\end{itemize}

We did not compare to the Joint Exponential Mechanism \cite{joseph2021joint}, due to its prohibitive  cost for large $k$. Apart from its runtime, it is in any case very similar to $\CANONICAL$ (see  Section~\ref{sec:related}). We implemented all methods in Python\FORK{}{ (our code is in the  supplementary material).}{.} More details \REVISE{on the setup are}{} in Section~\ref{sec:repl} of  the Appendix \FORK{and on \url{https://github.com/shekelyan/dptopk}}{}.

\subsection{Top-$k$: Real-World data}

In this experiment, $\vec{x} \in \mathbb{R}^d$ is based on half a million users who rated $d = 17700$ movies in the Netflix prize dataset (see Introduction) between $1$ and $5$. Each component $\vec{x}_i$ is equal to the number of users that gave the movie $i$ a 5/5 rating such that the sensitivity is $\frac{1}{2}$ (cf. Theorem~\ref{thm:asymmetric} in Appendix). As can be seen in Figure~\ref{fig:netflix} the new methods  $\CANONICAL$ / $\CANONICAL_{\gamma = 1}$ almost certainly return the correct top-$1000$ with a privacy budget of $\eps \le 1$, whereas classical methods require an up to $81 \times$ larger privacy budget to achieve the same feat. For $k = 100$, the privacy budget of our methods is \REVISE{smaller by $\times 34$}{$34 \times$ smaller} and for $k = 10$ it is \REVISE{by $\times 6$}{$6 \times$ smaller}. In \REVISE{}{the }Appendix, we report similar results for five additional real-world datasets (cf. Table~\ref{tab:datasets} and Figures~\ref{fig:patent},\ref{fig:searchlogs},\ref{fig:medcost},\ref{fig:income},\ref{fig:hepth}). In conclusion, the new methods show vast improvements.

\subsection{Top-$k$: Synthetic Data}

\begin{figure}
\centering
\SINGLEDOUBLE{\includegraphics[width=0.75\textwidth]{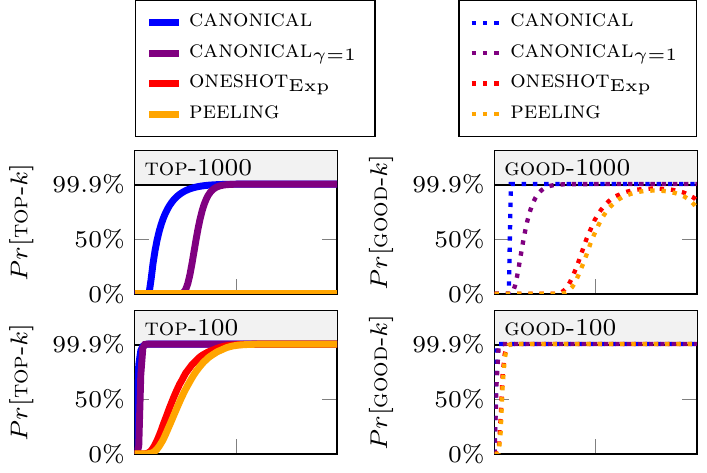}}{\includegraphics[width=0.45\textwidth]{plots/z4top1000.pdf}}
\caption{For a fixed privacy budget $\eps = 1$, the data distribution of $\vec{x} \in \mathbb{R}^d$ with $d = 10^4$ is varied from uniform values ($s = 0$) to linear values ($s = 1$) for $k \in \{10,100,1000\}$. \label{fig:zipfA}}
\vspace{0.25cm}
\centering
\SINGLEDOUBLE{\includegraphics[width=0.75\textwidth]{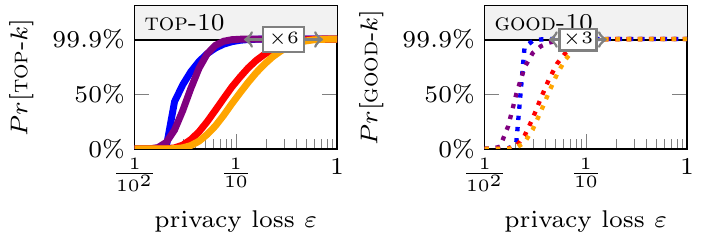}}{\includegraphics[width=0.45\textwidth]{plots/f4good10.pdf}}
\caption{Follow-up to Figure~\ref{fig:zipfA} with $s = 0.04$ and varying $\eps$ (for $\eps = 1$ and $s \ge 0.04$ all compared mechanisms have $Pr[\GOOD{10}] \ge Pr[\TOP{10}] \ge 99.9\%$) \label{fig:zipfB}.}
\end{figure}
\vspace{-0.25cm}
We generated $\vec{x} \in \mathbb{R}^d$ with $d = 10^4$ directly by Zipf's law, i.e., $f(i) \propto {i^s}$ and $\vec{x}_i = (1.5 \cdot 10^8) f(i)$. 
Figure~\ref{fig:zipfA} shows the results for varying parameter $s$ of Zipfian distribution, which controls how challenging the distribution is; scores are uniform when $s=0$. $\CANONICAL$ and $\CANONICAL_{\gamma = 1}$ appear clearly superior to the existing mechanisms, being able to sample the top-$100$ or top-$100$ with a probability many times larger. For top-$10$, the difference is not obvious for $\eps=1$, but it becomes much larger for smaller $\eps$ (see Figure~\ref{fig:zipfB}).
Besides top-$k$, we considered a ``good''-$k$ subset as one that replaces at most half of the top-$k$ elements with elements from the top-$\lfloor \frac{3}{2}k \rfloor$ without touching the top-$\lceil \frac{1}{100}k \rceil$, i.e., $y \in \mathcal{C}_{h,t}$ with $h \ge \frac{1}{100}k$ and $t \le \frac{3}{2}k$ (cf. Definition~\ref{def:subsetclass}). The new mechanisms again outperform classical ones, particularly for large $k$ and small $\eps$ which is more demanding, but with a smaller margin than for top-$k$.

\subsection{Runtimes}


\newcommand{\microsec}[1]{\ensuremath{#1 \,\mu\text{s}}}

In our experiments with $2776$ runtime measurements\footnote{Running on a laptop with 3.2Ghz Apple M1 processor and Python {\tt{3.99}} interpreter using one single thread/core. The reported runtimes do not include sorting the data vector $\vec{x}$, which in our experiments took less than \microsec{(d \log_2 d)}  ($\microsec{1} = 10^{-6}$ secs).}, $\CANONICAL$ was as fast as $\PEELING$ (average $< 1$s) and $\CANONICAL_{\gamma=1}$ was as fast as $\ONESHOT$ (average $< 3$ms). 
As all methods are simple, the recorded runtimes can be predicted with $< 80\%$ relative error by \microsec{(d \cdot k)} for $\tilde{O}(dk)$ methods $\PEELING$ / $\CANONICAL$ and by \microsec{(d)}  for $\tilde{O}(d)$ methods $\ONESHOT$ / $\CANONICAL_{\gamma = 1}$. 

\section{Discussion}

We investigated three questions in this work. If it is possible to unify existing discrete selection mechanisms with an additive noise framework, if it is possible to operate selection mechanisms efficiently over subsets as an immediate selection domain, and if there are theoretical differences between approaches that select a $k$-subset directly or independently in $k$ steps (\textsc{Peeling}). The Lipschitz mechanism conclusively answers the first question, the Canonical Lipschitz mechanism (\CANONICAL{}) is itself an affirmative answer to the second question, and our analysis shows an $\Omega( \log k)$ factor improvement of \CANONICAL{} (our direct approach) over \textsc{Peeling} when $\frac{d}{k}=O(1)$, i.e., when $d$ does not grow faster than $k$. Experimental results also indicate clear practical benefits, i.e., being able to quickly and reliably obtain high utility subsets with a far smaller privacy budget. 

Open questions include if the Lipschitz mechanism could be generalized to \REVISE{continuous}{non-finite} selection domains, if other loss functions over subsets could yield similar efficiency and utility results, and if the theoretical analysis could be extended beyond Gumbel noise, which is used to instantiate the \REVISE{E}{e}xponential mechanism. Specifically, it appears all noise distributions perform fairly similarly. Also,  it would be interesting to see if unification, efficiency, and theoretical results in the same spirit could be achieved for other differential privacy (DP) notions such as approximate DP.

\bibliography{literature}

\begin{thebibliography}{24}
\providecommand{\natexlab}[1]{#1}
\providecommand{\url}[1]{\texttt{#1}}
\expandafter\ifx\csname urlstyle\endcsname\relax
  \providecommand{\doi}[1]{doi: #1}\else
  \providecommand{\doi}{doi: \begingroup \urlstyle{rm}\Url}\fi

\bibitem[Asi \& Duchi(2020{\natexlab{a}})Asi and Duchi]{asi2020instance}
Asi, H. and Duchi, J.~C.
\newblock Instance-optimality in differential privacy via approximate inverse
  sensitivity mechanisms.
\newblock \emph{Advances in Neural Information Processing Systems}, 33,
  2020{\natexlab{a}}.

\bibitem[Asi \& Duchi(2020{\natexlab{b}})Asi and Duchi]{asi2020near}
Asi, H. and Duchi, J.~C.
\newblock Near instance-optimality in differential privacy.
\newblock \emph{arXiv preprint arXiv:2005.10630}, 2020{\natexlab{b}}.

\bibitem[Balog et~al.(2017)Balog, Tripuraneni, Ghahramani, and
  Weller]{balog2017lost}
Balog, M., Tripuraneni, N., Ghahramani, Z., and Weller, A.
\newblock Lost relatives of the gumbel trick.
\newblock In \emph{ICML}, pp.\  371--379, 2017.

\bibitem[Beimel et~al.(2016)Beimel, Nissim, and Stemmer]{nissim}
Beimel, A., Nissim, K., and Stemmer, U.
\newblock Private learning and sanitization: Pure vs. approximate differential
  privacy.
\newblock \emph{Theory Comput.}, 12\penalty0 (1):\penalty0 1--61, 2016.
\newblock \doi{10.4086/toc.2016.v012a001}.
\newblock URL \url{https://doi.org/10.4086/toc.2016.v012a001}.

\bibitem[Bennett \& Lanning(2007)Bennett and Lanning]{netflix_data_paper}
Bennett, J. and Lanning, S.
\newblock The netflix prize.
\newblock In \emph{In KDD Cup and Workshop in conjunction with KDD}, 2007.

\bibitem[Carvalho et~al.(2020)Carvalho, Wang, Gondara, and Miao]{uai_21}
Carvalho, R.~S., Wang, K., Gondara, L., and Miao, C.
\newblock Differentially private top-k selection via stability on unknown
  domain.
\newblock In \emph{UAI}, volume 124, pp.\  1109--1118, 2020.

\bibitem[Cesar \& Rogers(2021)Cesar and Rogers]{cesar2021bounding}
Cesar, M. and Rogers, R.
\newblock Bounding, concentrating, and truncating: Unifying privacy loss
  composition for data analytics.
\newblock In \emph{Algorithmic Learning Theory}, pp.\  421--457. PMLR, 2021.

\bibitem[Chaudhuri et~al.(2014)Chaudhuri, Hsu, and Song]{chaudhuri2014large}
Chaudhuri, K., Hsu, D., and Song, S.
\newblock The large margin mechanism for differentially private maximization.
\newblock In \emph{NIPS}, pp.\  1287--1295, 2014.

\bibitem[Ding et~al.(2021)Ding, Kifer, Steinke, Wang, Xiao, Zhang,
  et~al.]{ding2021permute}
Ding, Z., Kifer, D., Steinke, T., Wang, Y., Xiao, Y., Zhang, D., et~al.
\newblock The permute-and-flip mechanism is identical to report-noisy-max with
  exponential noise.
\newblock \emph{arXiv preprint arXiv:2105.07260}, 2021.
\newblock \url{https://arxiv.org/pdf/2105.07260.pdf}.

\bibitem[Dong et~al.(2020)Dong, Durfee, and Rogers]{dong2020optimal}
Dong, J., Durfee, D., and Rogers, R.
\newblock Optimal differential privacy composition for exponential mechanisms.
\newblock In \emph{International Conference on Machine Learning}, pp.\
  2597--2606. PMLR, 2020.

\bibitem[Durfee \& Rogers(2019)Durfee and Rogers]{durfee2019practical}
Durfee, D. and Rogers, R.
\newblock Practical differentially private top-k selection with
  pay-what-you-get composition.
\newblock In \emph{Proceedings of the 33rd International Conference on Neural
  Information Processing Systems}, pp.\  3532--3542, 2019.

\bibitem[Dwork \& Roth(2013)Dwork and Roth]{dwork2013algorithmic}
Dwork, C. and Roth, A.
\newblock The algorithmic foundations of differential privacy.
\newblock \emph{Theoretical Computer Science}, 9\penalty0 (3-4):\penalty0
  1--277, 2013.
\newblock
  \url{https://projects.iq.harvard.edu/files/privacytools/files/the_algorithmic_foundations_of_differential_privacy_1.pdf}.

\bibitem[Dwork et~al.(2014)Dwork, Roth, et~al.]{dwork2014algorithmic}
Dwork, C., Roth, A., et~al.
\newblock The algorithmic foundations of differential privacy.
\newblock \emph{Found. Trends Theor. Comput. Sci.}, 9\penalty0 (3-4):\penalty0
  211--407, 2014.

\bibitem[Fujiwara et~al.(2013)Fujiwara, Nakatsuji, Shiokawa, Mishima, and
  Onizuka]{aaai_13}
Fujiwara, Y., Nakatsuji, M., Shiokawa, H., Mishima, T., and Onizuka, M.
\newblock Fast and exact top-k algorithm for pagerank.
\newblock In \emph{AAAI}, 2013.

\bibitem[Ilyas et~al.(2008)Ilyas, Beskales, and Soliman]{top_k_survey}
Ilyas, I.~F., Beskales, G., and Soliman, M.~A.
\newblock A survey of top-k query processing techniques in relational database
  systems.
\newblock \emph{ACM Comput. Surv.}, 40\penalty0 (4), October 2008.
\newblock ISSN 0360-0300.
\newblock \doi{10.1145/1391729.1391730}.
\newblock URL \url{https://doi.org/10.1145/1391729.1391730}.

\bibitem[Joseph et~al.(2021)Joseph, Gillenwater, Ribero,
  et~al.]{joseph2021joint}
Joseph, M., Gillenwater, J., Ribero, M., et~al.
\newblock A joint exponential mechanism for differentially private top-k set.
\newblock In \emph{NeurIPS 2021 Workshop Privacy in Machine Learning}, 2021.

\bibitem[Lurie \& Hartley(1972)Lurie and Hartley]{lurie1972machine}
Lurie, D. and Hartley, H.
\newblock Machine-generation of order statistics for monte carlo computations.
\newblock \emph{The American Statistician}, 26\penalty0 (1):\penalty0 26--27,
  1972.

\bibitem[McKenna \& Sheldon(2020)McKenna and Sheldon]{mckenna2020permute}
McKenna, R. and Sheldon, D.~R.
\newblock Permute-and-flip: A new mechanism for differentially private
  selection.
\newblock \emph{NeurIPS}, 33, 2020.

\bibitem[McSherry \& Talwar(2007)McSherry and Talwar]{mcsherry2007mechanism}
McSherry, F. and Talwar, K.
\newblock Mechanism design via differential privacy.
\newblock In \emph{FOCS}, pp.\  94--103, 2007.

\bibitem[McSherry(2009)]{mcsherry2009privacy}
McSherry, F.~D.
\newblock Privacy integrated queries: an extensible platform for
  privacy-preserving data analysis.
\newblock In \emph{Proceedings of the 2009 ACM SIGMOD International Conference
  on Management of data}, pp.\  19--30, 2009.

\bibitem[Medina \& Gillenwater(2020)Medina and Gillenwater]{medina2020duff}
Medina, A.~M. and Gillenwater, J.
\newblock Duff: A dataset-distance-based utility function family for the
  exponential mechanism.
\newblock \emph{arXiv preprint arXiv:2010.04235}, 2020.

\bibitem[Narayanan \& Shmatikov(2009)Narayanan and Shmatikov]{netflix_deanon}
Narayanan, A. and Shmatikov, V.
\newblock De-anonymizing social networks.
\newblock In \emph{{S}{\&}{P} 2009}, pp.\  173--187, 2009.

\bibitem[Qiao et~al.(2021)Qiao, Su, and Zhang]{qiao2021oneshot}
Qiao, G., Su, W., and Zhang, L.
\newblock Oneshot differentially private top-k selection.
\newblock In Meila, M. and Zhang, T. (eds.), \emph{Proceedings of the 38th
  International Conference on Machine Learning}, volume 139 of
  \emph{Proceedings of Machine Learning Research}, pp.\  8672--8681. PMLR,
  18--24 Jul 2021.
\newblock URL \url{https://proceedings.mlr.press/v139/qiao21b.html}.

\bibitem[Wu et~al.(2007)Wu, Kumar, Ross~Quinlan, Ghosh, Yang, Motoda,
  McLachlan, Ng, Liu, Yu, Zhou, Steinbach, Hand, and Steinberg]{kais_top10}
Wu, X., Kumar, V., Ross~Quinlan, J., Ghosh, J., Yang, Q., Motoda, H.,
  McLachlan, G.~J., Ng, A., Liu, B., Yu, P.~S., Zhou, Z.-H., Steinbach, M.,
  Hand, D.~J., and Steinberg, D.
\newblock Top 10 algorithms in data mining.
\newblock \emph{Knowl. Inf. Syst.}, 14\penalty0 (1):\penalty0 1–37, December
  2007.
\newblock ISSN 0219-1377.
\newblock \doi{10.1007/s10115-007-0114-2}.
\newblock URL \url{https://doi.org/10.1007/s10115-007-0114-2}.

\end{thebibliography}
\bibliographystyle{icml2022}

\newpage
\appendix

\section{Appendix}

\subsection{Implementation} \label{sec:suppimpl}

\begin{theorem}\REVISE{[implementation]}{} \label{thm:implementation}
The Canonical Lipschitz Mechanism for top-$k$ can be sampled in $O(dk)$ time \REVISE{if all $d$ scores are pre-sorted}{for $d$ pre-sorted scores}.
\end{theorem}
\begin{proof}

\REVISE{The Canonical Lipschitz mechanism can be implemented to release a random subset from the}{The mechanism releases a subset from } class $\mathcal{C}_{H,T}$ \REVISE{}{(cf. Definition~\ref{def:subsetclass})} with $(H,T)$ equal to:

$$\argmax_{(h,t)} \left\{ \frac{ \,(1-\gamma)\top{h+1}-(\gamma)\top{t} }{-{2/ \eps}}  +F^{-1}( {U_{h,t}}^{1/|\mathcal{C}_{h,t}|} ) \right\}$$

with $|\mathcal{C}_{k-1,k}| = 1$ and $|\mathcal{C}_{h,t}| = {{t-h-2}  \choose {k-1-h}}$ for $h \in \{0, \ldots, k-1\}, t \in \{k+1, \ldots, d\}$. 

The binomial coefficients for each $t$ can be computed in $O(k)$ by starting with $h = k-1$ and ${{t-h-2} \choose {0}} = 1$  and decreasing $h$ until $h = 0$. In each step, $h$ decreases by $1$ and ${a \choose b} = \frac{a}{b}{ {a-1} \choose {b-1}}$ can be used to update the binomial coefficient. The number of subsets in each class $\mathcal{C}_{h,t}$ is equal to the distinct number of possibilities for the body $\mathcal{B}$, which is equal to the number of $(k-1-h)$-subsets out of $|\mathcal{B}| = |\{\topitem{h+2},\ldots, \topitem{t-1}\}|=({t-h-2})$ items (counted as $1$ if $|\mathcal{B}| = 0$). All subsets in the class $\mathcal{C}_{h,t}$ from Definition~\ref{def:subsetclass} have the same loss value. As the time complexity of implementations  hinges upon the distinct number of loss values, one can see that there is the class $\mathcal{C}_{k-1,k}$ with the top-$k$ and otherwise $h \in \{0, \ldots, k-1\}$ and $t \in \{k+1, \ldots, d\}$. Thus, the total number of classes is %

$1+k(d-k) = 1+dk-k^2 = O(dk)$.
\end{proof}

\REVISE{}{
The time can be reduced to $O(d)$ for $\gamma=1$ as only each $t \in \{k, \ldots, d\}$ has a distinct loss value. In this case $\sum_{h = 0}^{k-1} |\mathcal{C}_{h,t}| = {{t-1} \choose {k-1}}$ which can be also computed via $O(1)$ time updates by considering that for $n,k \in \mathbb{N}$:

\begin{center}
\begin{tabular}{lcl}
\toprule
 ${{n} \choose {n}} = { {k} \choose {k}} = 1$ & \ \ & ${{n} \choose {1}} = {{n} \choose {n-1}} = n$ \\
\midrule
${ {n} \choose {k-1}}= \frac{k}{n-k+1} {n \choose k }$ & &
 ${ {n} \choose {k+1}}= \frac{n-k}{k+1} {n \choose k }$ \\
\midrule
${ {n-1} \choose {k}}= \frac{n-k}{n} {n \choose k }$ & &
 ${ {n+1} \choose {k}}= \frac{n+1}{n+1-k} {n \choose k }$ \\
\midrule
 ${ {n-1} \choose {k-1}}= \frac{k}{n} {n \choose k }$ & &
 ${ {n+1} \choose {k+1}}= \frac{n+1}{k+1} {n \choose k }$ \\
\bottomrule
\end{tabular}
\end{center}
}
\SetKwInput{KwInput}{Input}
\SetKwInput{KwOutput}{Output}

\begin{algorithm2e}
\DontPrintSemicolon
\REVISE{}{\KwInput{scoring function $f : \mathbb{N} \times \mathbb{X} \rightarrow \mathbb{R}$ over $d$ items, subset size $k \in \{1, \ldots, d-1\}$ and  $\eps \in \mathbb{R}_{\ge 0}$}}

Let $\Delta_{f}$ be the sensitivity of $f$. \;
Let $\vec{x}_i = f(i \mid \hat{x})/\Delta_f$ for any $i \in \{1, \ldots, d\}$. \;

\REVISE{}{Let $\vec{x}_{[1]} \ge \ldots \ge \vec{x}_{[d]}$ sort $\vec{x}$ \tcp{in $O(d\log d)$}}

\If{$\gamma < 1$}{

\REVISE{\tcc{$\CANONICAL$ ($\gamma \in [0,1]$) with $O(dk)$ runtime}}{
\tcc{$\CANONICAL$ sampling}
\tcc{in $O(dk)$ time and $O(1)$ space}
}

Let $U_{h,t} \sim Unif(0,1)$ be i.i.d. for any $h \in \{0, \ldots, k-1\}, t \in \{k, \ldots, d\}$ \;
Initiate $H = k-1$ and $T = k$. \;
Let $\eps_1 = (1-\gamma)\eps$ and $\eps_2 = \gamma \cdot \eps$ \REVISE{}{s.t. $\eps_1+\eps_2 = \eps$}. \;
\REVISE{Set $H = k-1$, $T = k$ and}{Initiate } $v = \frac{\eps_2-\eps_1}{2} \top{k} + F^{-1}(U_{H,T})$ \;
\ForEach{$t \in \{k+1, \ldots, d\}$}{
    \REVISE{}{Initiate} $m = 0$\hfill \tcp{$m = \log {{t-(h+2)} \choose {k-(h+1)}}$ } 
    
    \REVISE{}{Initiate} $h = k-1$ \;
    \While{$h \ge 0$}{
        \If{$k-(h+1) > 0$}{
            \REVISE{Set}{Update} $m = m+ \REVISE{\log(\,t-(h+1)\,)-\log(\,k-(h+1)\,)}{\log \frac{t-(h+1}{k-(h+1)} } $ \;
        }
        \REVISE{
        Let $v' = { \frac{\eps_2}{2} \top{t}  - \frac{ \eps_1}{2}\top{h+1} + F^{-1}({U_{h,t}}^{\exp(-m)})}$ \;}{
            Let $X = F^{-1}({U_{h,t}}^{\exp(-m)})$ \;
            Let $v' = \frac{\eps_2}{2} \top{t}  - \frac{\eps_1}{2} \top{h+1} + X$ \;
        }
        \If{$v' > v$}{
            \REVISE{Set}{Update} $v = v'$, $H = h$ and $T = t$ \;
        }
       \REVISE{}{Update} $h = h-1$\;
    }
}
Report random subset in $\mathcal{C}_{H,T}$.
} \Else{

\REVISE{
\tcc{$\CANONICAL_{\gamma = 1}$ with $O(d)$ runtime}
}{
\tcc{$\CANONICAL_{\gamma = 1}$ sampling}
\tcc{ in $O(d)$ time and $O(1)$ space}
}

Let $U_\REVISE{i}{t} \sim Unif(0,1)$ be i.i.d. for \REVISE{any}{} $\REVISE{i}{t} \in \{\REVISE{1}{k}, \ldots, d\}$. \;
Initiate $T = k$ and $v = \top{k} + F^{-1}(U_{T})$ \;
Initiate $m = 0$ \hfill\tcp{$m = \log {{t-1} \choose {k-1}}$ }
Initiate $t = k+1$ \;
\While{$t \le d $}{
    \REVISE{Set}{Initiate}  $m = m+ \REVISE{\log(\,t-1\,)-\log(\, (t-1)-(k-1)\, )}{\log \frac{t-1}{(t-1)-(k-1)}} $ \;
    \If{ $(t-1) \ge (k-1)$ }{
    
        \REVISE{
    
        Let $v' = \top{t} \eps_2 - \top{h+1} \eps_1{ \eps \cdot \top{t}} + F^{-1}({U_t}^{\exp(-m)})$ \;
        
        }{  Let $X = F^{-1}({U_t}^{\exp(-m)})$ \; Let $v' = \frac{\eps}{2} \top{t} + X$\;}
        
        \If{$v' > v$}{
            \REVISE{Set}{Update} $v = v'$ and $T = t$ \;
        }
    }
    \REVISE{Set}{Update} $t = t+1$ \;
}
\If{$T = k$}{ 
    Report $\{1, \ldots, k\}$ \;
} \Else {
    Report random subset in $(\mathcal{C}_{0,T} \cup \mathcal{C}_{1,T} \cup \ldots \cup \mathcal{C}_{k-1,T})$.
}

}
\caption{Canonical Lipschitz Mechanism for top-$k$ \REVISE{}{(with $F^{-1}$ for additive noise generation and $\gamma \in [0,1]$ )} \label{alg:canonical}}
\end{algorithm2e}

\begin{algorithm2e}
\DontPrintSemicolon
\REVISE{}{\KwInput{scoring function $f : \mathbb{N} \times \mathbb{X} \rightarrow \mathbb{R}$ over $d$ items, subset size $\kappa \in \{1, \ldots, d-1\}$ and  $\eps \in \mathbb{R}_{\ge 0}$}}
Let $\vec{x}_i = f(i \mid \hat{x})/\Delta_f$ for any $i \in \{1, \ldots, d\}$. \;

Let $U_\REVISE{i}{t} \sim Unif(0,1)$ be i.i.d. for any $\REVISE{i}{t} \in \{\REVISE{1}{k}, \ldots, d\}$. \;

Maintain \REVISE{an initally empty}{} heap structure \REVISE{with the}{for} to\REVISE{o}{p}-$\kappa$ based on scores \;

\ForEach{$i \in \{1, \ldots, d\}$}{

    Let $Z_i = \frac{\eps}{ 2\kappa\Delta_f}\vec{x}_i + F^{-1}(U_i)$ \;
    Consider $i$ for top-$\kappa$ based on noisy score $Z_i$ \;
}

Report top-$\kappa$.
\caption{\REVISE{General}{(General)} Lipschitz Mechanism \label{alg:lipschitz} \REVISE{}{(with $F^{-1}$ for additive noise generation)}}
\end{algorithm2e}

\REVISE{To compute $F^{-1}( U^{\exp(-m)} ) $ one can either use arbitrary precision libraries, or in case that $F^{-1}$ is the inverse distribution function of the Gumbel distribution one can exploit:}{When instantiating the exponential mechanism, subset probabilities 
$Pr(y \in \mathcal{C}_{h,t}) \propto \exp( \frac{ \,(1-\gamma)\top{h+1}-(\gamma)\top{t} }{-{2/ \eps}} ) $. For sampling large numbers can be avoided by taking $F^{-1}( U^{\exp(-m)} )$ into the logspace:}

\begin{lemma}
Let $F^{-1}(p) = -\log(-\log(p))$, \REVISE{then:
$$F^{-1}( U^{\exp(-m)} ) = m + F^{-1}(U)$$}{then $F^{-1}( U^{\exp(-m)} ) = m + F^{-1}(U)$}
\end{lemma}
\begin{proof}
\begin{align*}
F^{-1}( U^{\exp(-m)} ) &= \, -\log(-\log(U^{\exp(-m)})) \\
&= \, -\log(-\log(U) \exp(-m)) \\
&= \, -\log(-\log(U))-\log( \exp(-m)) \\
&= \, -\log(-\log(U))-(-m) \\
&= \, m-\log(-\log(U)) \\
&= \, m+F^{-1}(U) \\
\end{align*}
\end{proof}

\subsection{Replicability and Additional Experimental Results} \label{sec:repl}

Datasets are described in Table~\ref{tab:datasets}. The sensitivities of all featured (count-based) datasets are presumed to be $\Delta = \frac{1}{2}$ via the shifting trick (cf. Theorem~\ref{thm:asymmetric}).%
\REVISE{In addition to $\textsc{TOP}_k$, we distinguished two classes of high utility subsets based on the largest top-$h$ and the worst rank $h$ contained in a subset, as in Definition~\ref{def:subsetclass}. The first is $\textsc{GREAT}_k$ which mandates that every subset $y$ in this class contains at most top-$\lceil 0.1k \rceil$ items and the rest must come from the top-$\lfloor 1.1k \rfloor$, i.e. $y\in \mathcal{C}_{h,t}$ with $h\geq 0.1k$ and $t\leq 1.1k$. For example, for $k =  100$  this class mandates that every subset $y$ in it contains the top-$10$ items and the remaining $90$ items from the top-$110$. The second is $textsc{GOOD}_k$ which mandates that every subset $y$ in it must contain at least $\lceil 0.01k \rceil$ items from the top-$k$ and the rest from the top-$\lfloor 1.5k \rfloor $, i.e., $y\in $ with  $\mathcal{C}_{h,t}$ with $h\geq 0.01k$ and $t\leq 1.5k$. For example, for $k=100$ this class mandates that every subset $y$ in it contains at least the top-$1$ item and the remaining $99$ items from the top-$150$. 

\begin{align*}
\textsc{TOP}_k(y) &\Leftrightarrow y \in \mathcal{C}_{k-1,k} \Leftrightarrow y \in \OPT(\vec{x}) \\
\textsc{GREAT}_k(y) &\Leftrightarrow y \in \mathcal{C}_{h,t} \text{ with } { \, h \ge 0.1k \,} \text{ and }{ \, t \, \le 1.1k \, } \\
\textsc{GOOD}_k(y) &\Leftrightarrow y \in \mathcal{C}_{h,t} \text{ with } { \, h \ge 0.01k \,}  \text{ and } {  t \, \le 1.5k \, } \\%
\end{align*} }{

We aggregate $\mathcal{C}_{h,t}$ classes from Definition~\ref{def:subsetclass} into high utility predicates:

\begin{align*}
\textsc{TOP}\REVISE{_k}{}(y) &\Leftrightarrow y \in \mathcal{C}_{k-1,k} \Leftrightarrow y \in \OPT(\vec{x}) \\
\textsc{GREAT}\REVISE{_k}{}(y) &\Leftrightarrow y \in \mathcal{C}_{h,t} \text{ with } { \, h \ge \frac{1}{10}k \,} \text{ and }{ \, t \, \le k+\frac{k}{10} \, } \\
\textsc{GOOD}\REVISE{_k}{}(y) &\Leftrightarrow y \in \mathcal{C}_{h,t} \text{ with } { \, h \ge \frac{1}{100}k \,}  \text{ and } {  t \, \le k+\frac{k}{2} \, } \\%
\end{align*}

The predicate $\textsc{TOP}\REVISE{_k}{}(y)$ mandates $y$ to be the exact top-$k$. The predicate $\textsc{GREAT}\REVISE{_k}{}(y)$ mandates for $y$ the inclusion of the top-$\lceil \frac{1}{10}k \rceil$ and exclusion of of items outside of top-$\lfloor \frac{11}{10}k \rfloor$. For $k =  100$ this means inclusion of all top-$10$ items and the remaining $90$ items must come from top-$110$. The predicate $\textsc{GOOD}\REVISE{_k}{}(y)$ mandates for $y$ the inclusion of the top-$\lceil \frac{1}{100}k \rceil$ and exclusion of items outside of top-$\lfloor \frac{3}{2}k \rfloor$. 
}

Workflow of how each plot (with $\eps$ in $x$-axis) is generated:

\REVISE{
\begin{itemize}
\item We take a single vector $\vec{x}$ as an input to all mechanisms and fix a subset size $k$ (repeated for different $k \in \{10,100,1000\}$). 
\item We fix different values of $\eps$. For each value of $\eps$ we either calculate the probability distribution when possible ($\CANONICAL$, $\CANONICAL_{\gamma = 1})$ or estimate probabilities via Monte Carlo methods (using ideas from Section~\ref{sec:fastlipschitz}) with $10000$ generated subset classes ($\ONESHOT,\PEELING$) for each considered $\eps$ value. Additionally, we generate one subset for runtime measurements and to validate the probabilities.
\item Based on that we calculate for each $\eps$ the probabilities of selecting a subset that is \textsc{TOP}-$k$, \textsc{GREAT}-$k$ , \textsc{GOOD}-$k$ or neither (any \textsc{TOP}-$k$ is also \textsc{GREAT}-$k$ and any \textsc{GREAT}-$k$ is also \textsc{GOOD}-$k$). 
\item We plot (each plot has a fixed $k$ and $\vec{x}$ except Figure~\ref{fig:zipfA} and Figure~\ref{fig:zipfB} where the synthetic data distribution is varied) for each $\eps$ ($x$-value) the probability ($y$-value) of each considered utility class.
\end{itemize}
}{
\begin{itemize}
\item The vector $\vec{x}$ is fixed for one of the datasets and given as input to all mechanisms.
\item The subset size $k \in \{10,100,1000\}$ is fixed 
\item The privacy loss $\eps \in \mathbb{R}_{\ge}$ is then varied with sufficient precision for plotting purposes
\item For each data point $\vec{x},k,\eps$ either the probability distribution over $\mathcal{C}_{h,t}$ is computed ($\CANONICAL$, $\CANONICAL_{\gamma = 1})$ or probabilities are estimated via Monte Carlo methods (using ideas from Section~\ref{sec:fastlipschitz}) with $10000$ generated subset classes ($\ONESHOT,\PEELING$). Additionally, a subset is sampled for runtime measurements and validation purposes.
\item Each plot (for fixed $\vec{x},k$) then shows along the $x$-axis the $\eps$ (except Figure~\ref{fig:zipfA} and Figure~\ref{fig:zipfB} where the synthetic data distribution is varied) and along the $y$-axis the probability of either \textsc{TOP}-$k$, \textsc{GREAT}-$k$ or \textsc{GOOD}-$k$.
\end{itemize}
}

Implementation details:

\begin{itemize}
\item $\ONESHOT_{\text{Exp}}$ is implemented as the $\eps$-DP mechanism in Algorithm~\ref{alg:lipschitz} with $\kappa = k$.
\item \REVISE{$\PEELING$ samples the $\frac{\eps}{k}$-DP mechanism$k$ times without replacement as in Algorithm~\ref{alg:lipschitz} with $\kappa = 1$.}{$\PEELING$ samples $k$ times without replacement from the $\frac{\eps}{k}$-DP mechanism in Algorithm~\ref{alg:lipschitz} with $\kappa = 1$.}
\item $\CANONICAL$ as the $\eps$-DP mechanism in Algorithm~\ref{alg:canonical} with $\gamma = \frac{1}{2}$.
\item $\CANONICAL_{\gamma = 1}$ as the $\eps$-DP mechanism in Algorithm~\ref{alg:canonical} with $\gamma = 1$.
\end{itemize}

The inverse distribution function $F^{-1}$ used in the approaches (always standard distribution parameters):

\begin{tabular}{rll}
\toprule
Approach & $F^{-1}(p)$ & Distribution \\
\midrule
$\ONESHOT_{\text{Exp}}$ & $-\log(\REVISE{p}{1-p})$ & Exponential \\
$\PEELING$ & $-\log(-\log(p))$ & Gumbel \\
$\CANONICAL$ & $-\log(-\log(p))$ & Gumbel \\
$\CANONICAL_{\gamma = 1}$ & $-\log(-\log(p))$ & Gumbel \\
\bottomrule
\end{tabular}

Additional plots:

\begin{itemize}
\item Figure~\ref{fig:lipschitz} compares the Lipschitz mechanism with $\kappa = k \in \{1,10,100,1000\}$ for different choices of $F^{-1}$.

\item Figure~\ref{fig:netflixaddition} supplements Figure~\ref{fig:netflix} with \REVISE{$\textsc{GREAT}_k$}{$\textsc{GREAT}$-$k$} results.
\item Figures~\ref{fig:patent}, \ref{fig:searchlogs}, \ref{fig:medcost}, \ref{fig:income}, \ref{fig:hepth} replicate results from the paper for five additional datasets. \REVISE{The requires $\eps$ to reach good utility is sometimes large due to the small number of users participating in some of these datasets and is unrelated to the methods.}{Some datasets lack a sufficient number of users to reach good utility with any $\eps$-DP methods with $\eps < 1$.} In some rare instances the top-$k$ can be one out of many arbitrary subsets (we break ties arbitrarily for top-$k$), because we do not modify $\vec{x}$ to break ties between uniform values.
\end{itemize}

\begin{table*}
{
\begin{tabular}{lll}
\toprule
{\bf Dataset} $\vec{x} \in \mathbb{R}^d$ & {\bf Size $d$} & {\bf Description} (of data from which $\vec{x}$ is extracted)  \\
\midrule
$\textsc{netflix}$ & \REVISE{17700}{17770} & \SINGLEDOUBLE{\small}{}Netflix movie ratings (we count how many \REVISE{Netflix users who}{users} gave each movie a 5/5 rating) \\
$\textsc{patent}$ & 4096 & \SINGLEDOUBLE{\small}{}
Citation network among a subset of US
patents\\
$\textsc{searchlogs}$ & 4096 & \SINGLEDOUBLE{\small}{}Query logs for query ``Obama'' issued from Jan. 1, 2004 to Aug. 9, 2009.  \\
$\textsc{medcost}$ & 4096 & \SINGLEDOUBLE{\small}{}Personal medical
expenses based on a national home and hospice care survey \REVISE{from 2007}{}\\
$\textsc{income}$ & 4096 & \SINGLEDOUBLE{\small}{} ``Personal Income'' attribute of the IPUMS American community survey data \REVISE{from 2001-2011}{}\\
$\textsc{hepth}$ & 4096 & \SINGLEDOUBLE{\small}{}Citation network among high energy physics pre-prints
on arXiv \\
\bottomrule
\end{tabular}
}
\caption{Featured real-world data sets. 
The first dataset is the Netflix Prize dataset~\cite{netflix_data_paper}. We used as input to all approaches a data vector $\vec{x}\in \mathbb{R}^d$ with $d$ scores.  
The other datasets and their corresponding data vectors were obtained from~\cite{mckenna2020permute}. 
\label{tab:datasets}}
\end{table*}

\begin{figure*}
\centering
\begin{tabular}{lr}
Top-$1$ selection with $\kappa = k = 1$ & Top-$k$ oneshot selection with $\kappa = k \in \{10,100,1000\}$ \\[0.25cm]
\includegraphics[trim=5.25cm 19cm 11.5cm 4.25cm, clip,width=0.45\textwidth]{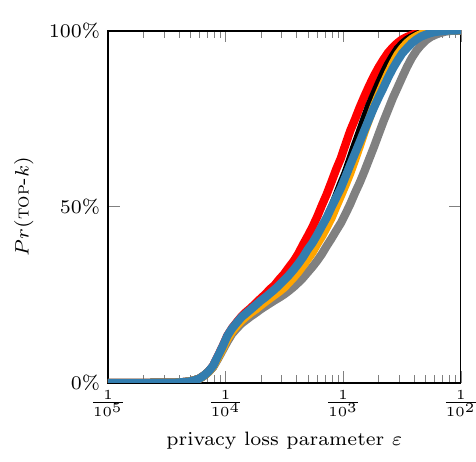} & \includegraphics[width=0.55\textwidth]{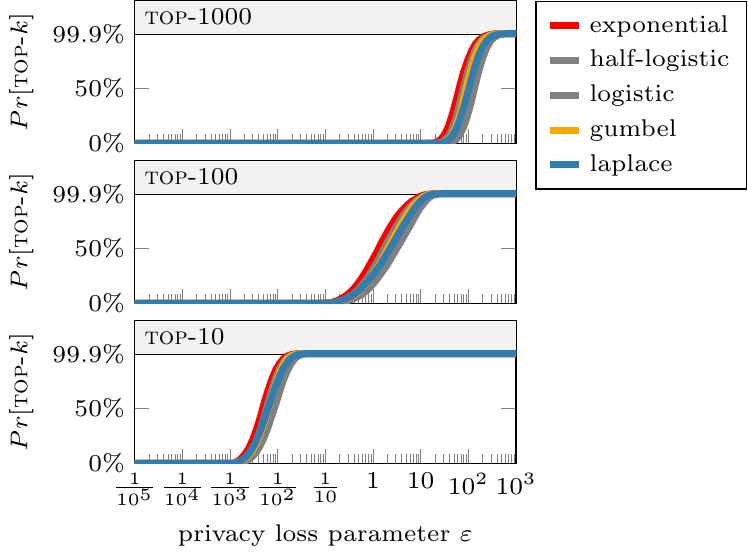}
\end{tabular}
\caption{Lipschitz mechanism with different noise distributions ($\vec{x} = \textsc{netflix} \in \mathbb{R}^{17770},k = \kappa \in \{1,10,100,1000\}, \eps \in \mathbb{R}_{> 0}$). Different noise distribution instantiate many selection mechanisms from the literature (cf. Section~\ref{sec:supplipschitz}) \label{fig:lipschitz}}
\end{figure*}

\newpage

\begin{figure*}
\includegraphics[width=0.95\textwidth]{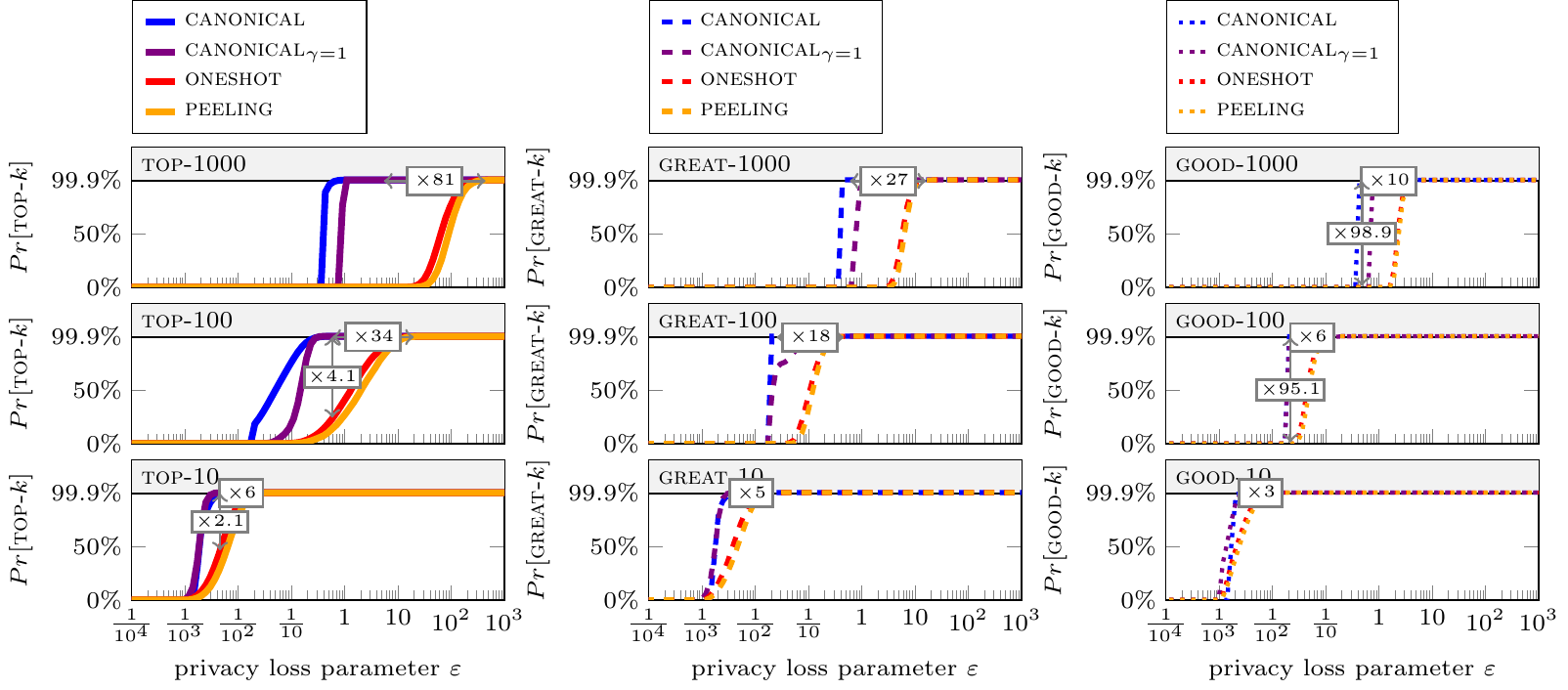}
\caption{Comparison ($\vec{x} = \textsc{netflix} \in \mathbb{R}^{17700},k \in \{10,100,1000\}, \eps \in \mathbb{R}_{> 0}$). }\label{fig:netflixaddition}
\end{figure*}

\begin{figure*}
\includegraphics[width=0.95\textwidth]{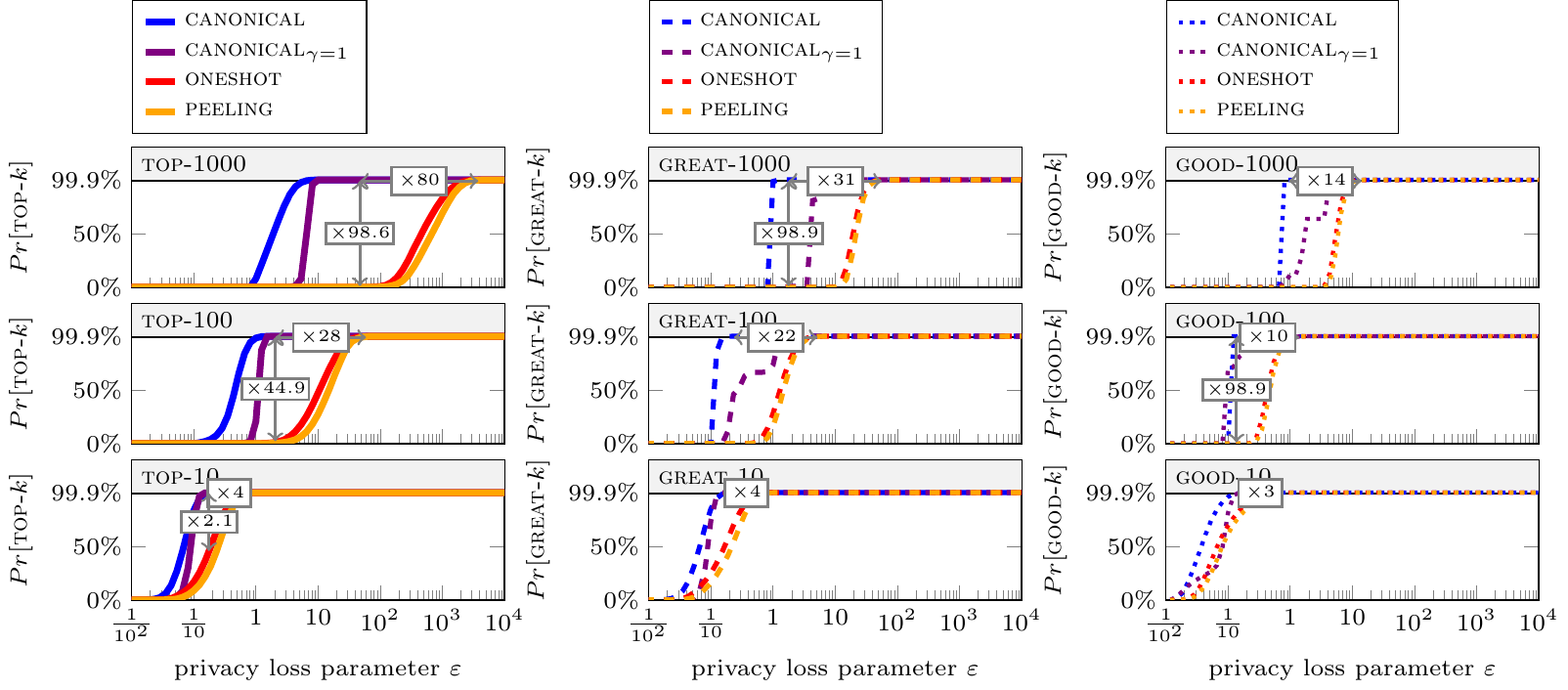}
\caption{Comparison ($\vec{x} = \textsc{patent} \in \mathbb{R}^{4096},k \in \{10,100,1000\}, \eps \in \mathbb{R}_{> 0}$)}\label{fig:patent}
\end{figure*}

\begin{figure*}
\includegraphics[width=0.95\textwidth]{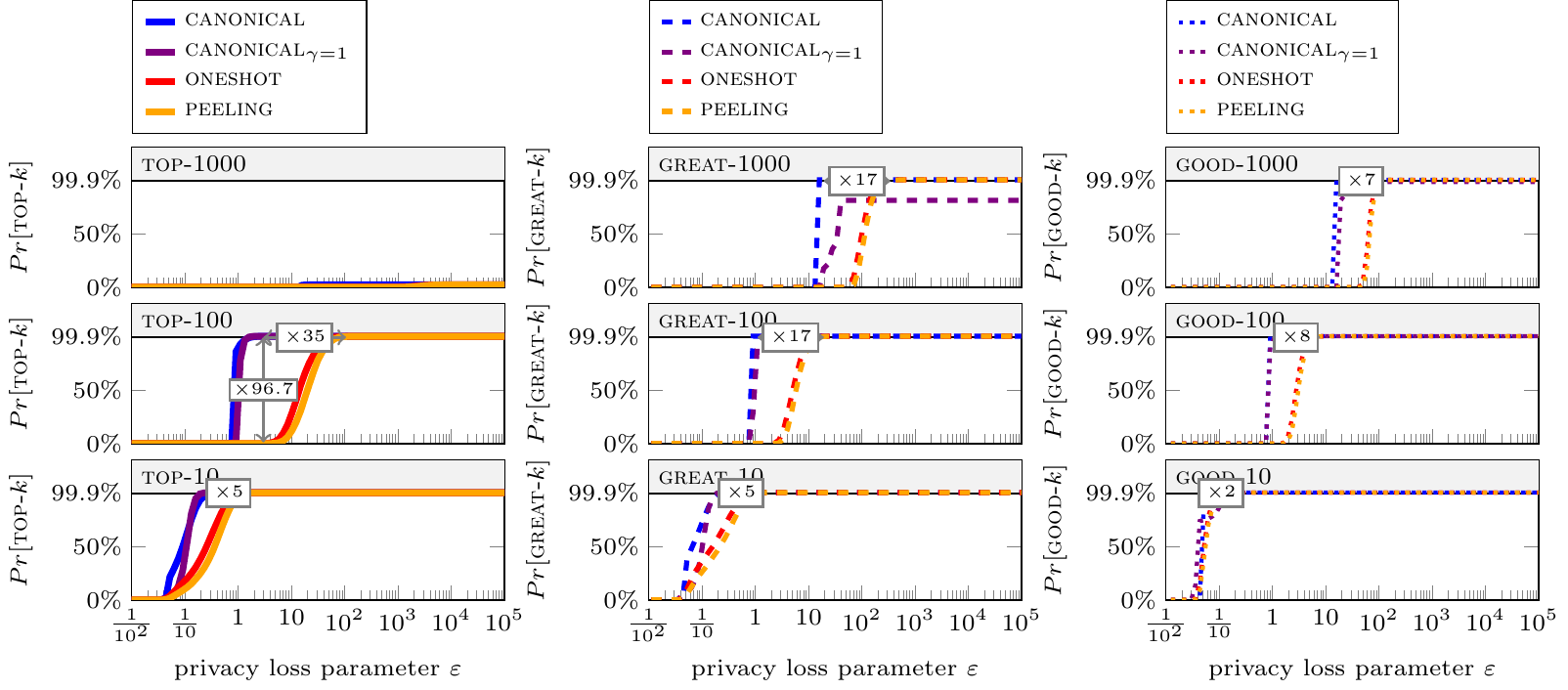}
\caption{Comparison ($\vec{x} = \textsc{searchlogs} \in \mathbb{R}^{4096},k \in \{10,100,1000\}, \eps \in \mathbb{R}_{> 0}$)}\label{fig:searchlogs}
\end{figure*}

\begin{figure*}
\includegraphics[width=0.95\textwidth]{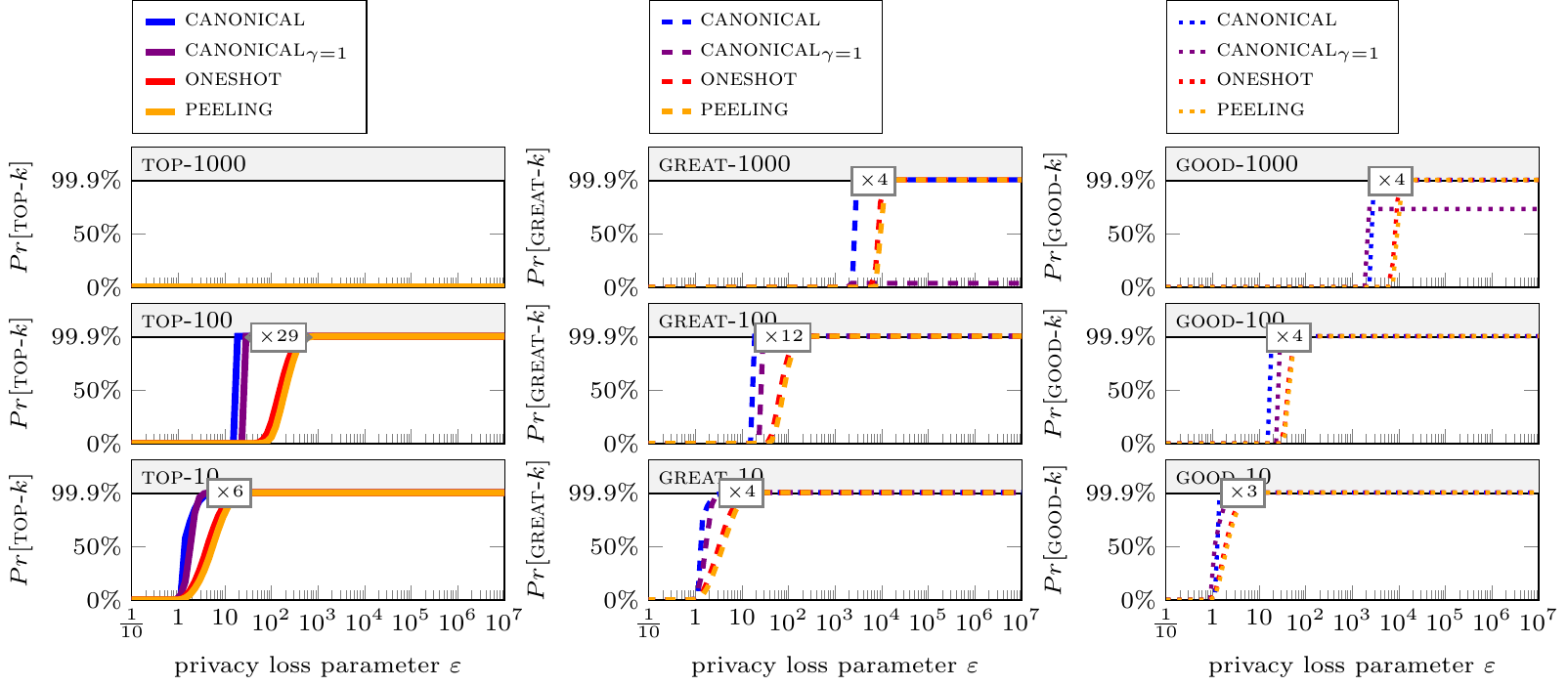}
\caption{Comparison ($\vec{x} = \textsc{medcost} \in \mathbb{R}^{4096},k \in \{10,100,1000\}, \eps \in \mathbb{R}_{> 0}$)}\label{fig:medcost}
\end{figure*}

\begin{figure*}
\includegraphics[width=0.95\textwidth]{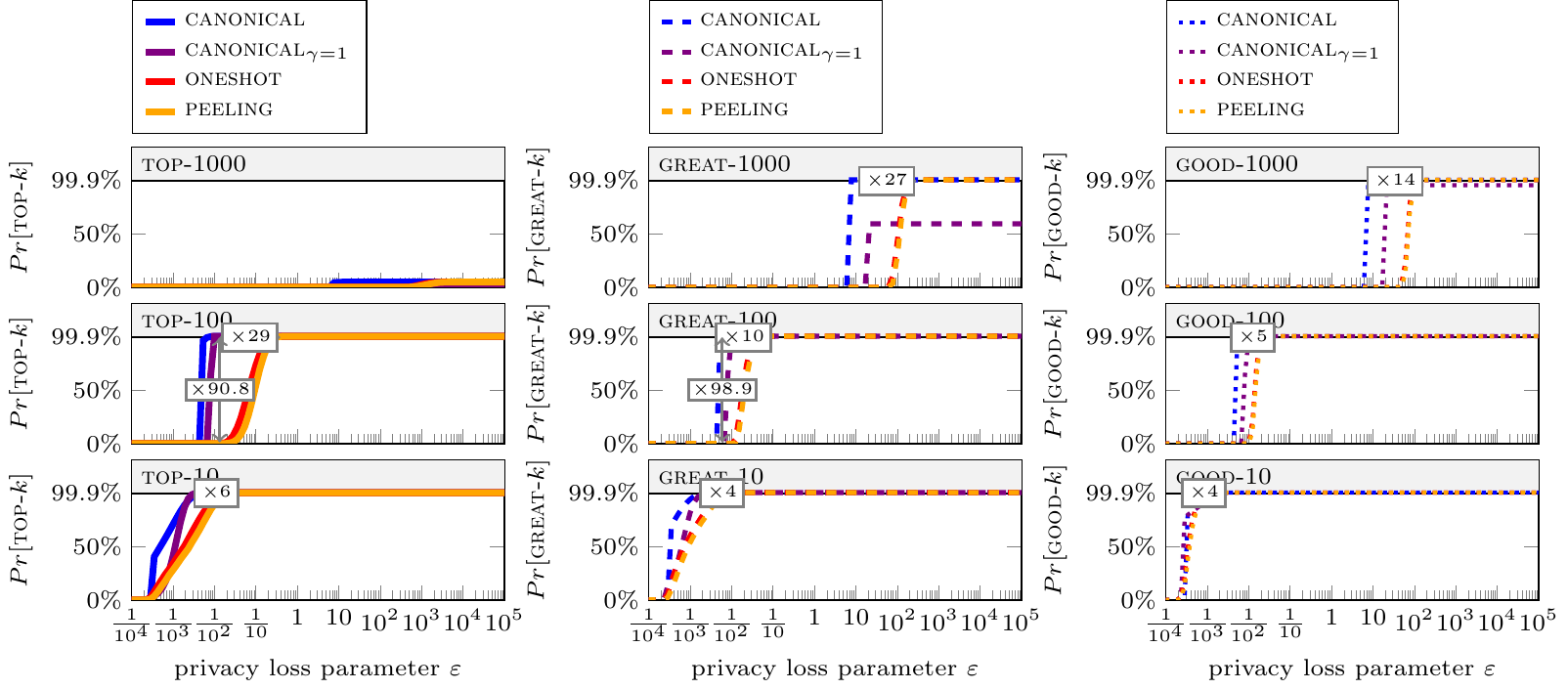}
\caption{Comparison ($\vec{x} = \textsc{income} \in \mathbb{R}^{4096},k \in \{10,100,1000\}, \eps \in \mathbb{R}_{> 0}$)}\label{fig:income}
\end{figure*}

\begin{figure*}
\includegraphics[width=0.95\textwidth]{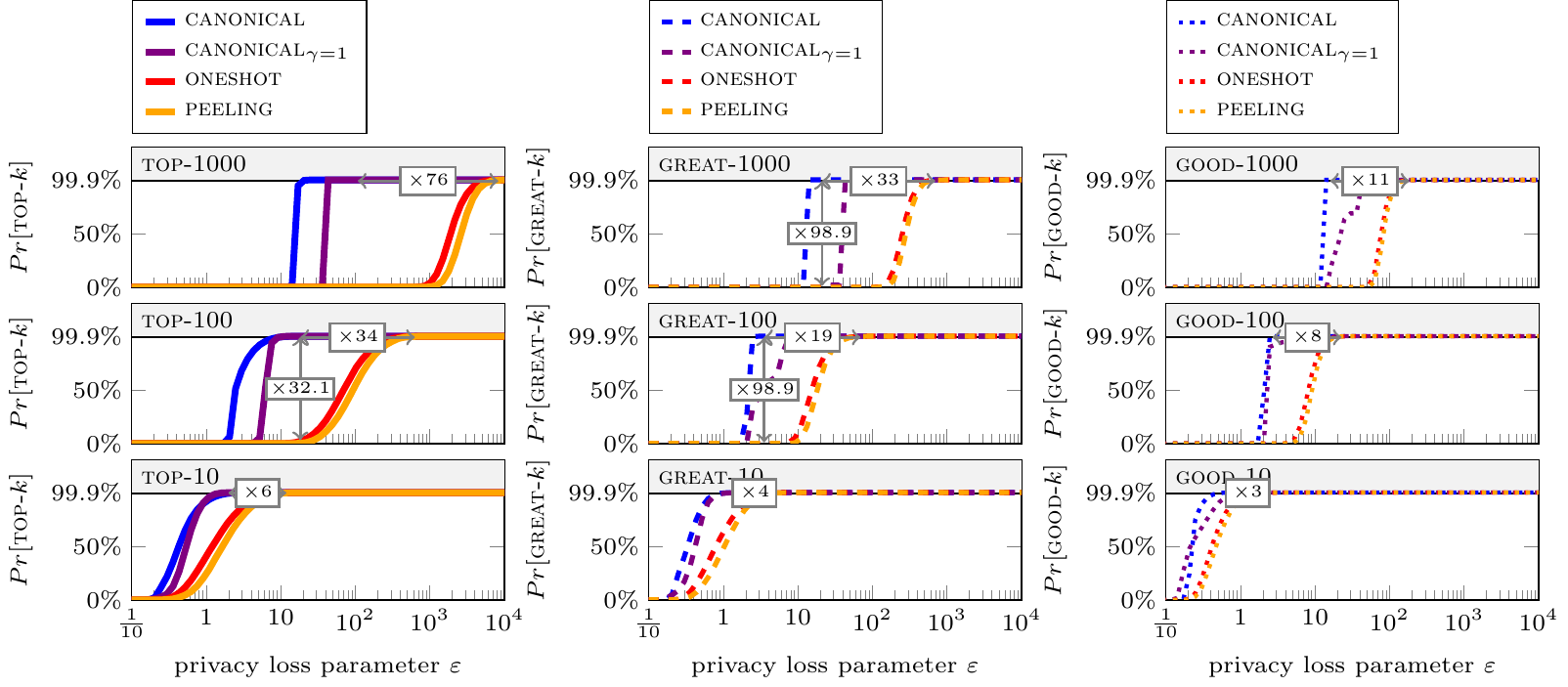}
\caption{Comparison ($\vec{x} = \textsc{hepth} \in \mathbb{R}^{4096},k \in \{10,100,1000\}, \eps \in \mathbb{R}_{> 0}$)}\label{fig:hepth}
\end{figure*}

\clearpage

\subsection{The Lipschitz Mechanism: Overview } \label{sec:supplipschitz}

The standard Exponential, Gumbel, Laplace, Logistic and Half-Logistic distribution are examples of distributions that satisfy the Lipschitz property mandated in the Lipschitz mechanism:

\begin{itemize}

\item Theorem~\ref{lem:lipschitzexponential}: The standard Exponential distribution with $F(x) = 1-\exp(-x)$ for $x \in \mathbb{R}_{\ge 0}$, $F(x) = 0)$ for $x \in \mathbb{R}_{< 0}$ and $F^{-1} = -\ln(1-p)$ for $p \in [0,1)$ satisfies
the Lipschitz condition from the Lipschitz mechanism.

\item Theorem~\ref{lipschitzgumbel}: The standard Gumbel distribution with $F(x) = \exp(- \exp(-x))$ for $x \in \mathbb{R}$ and $F^{-1} = -\ln( -\ln(p) )$ for $p \in [0,1)$ satisfies
the Lipschitz condition from the Lipschitz mechanism.

\item Theorem~\ref{lem:lipschitzlaplace}: The standard Laplace distribution with $F(x) = \frac{1}{2} \exp(-x)$ for $x \in \mathbb{R}_{\ge 0}$ and $F(x) =  1-\frac{1}{2} \exp(-|x|)$ for $x \in \mathbb{R}_{< 0}$ and $F^{-1}(p) = \sgn(p-\frac{1}{2}) \ln(1-|2p-1|)$ for $p \in [0,1)$ satisfies
the Lipschitz condition from the Lipschitz mechanism.

\item Theorem~\ref{lem:lipschitzhalflogistic}: The standard Half-Logistic distribution with $F(x) = \frac{1-\exp(-x)}{1+\exp(-x)}$ for $x \in \mathbb{R}_{\ge 0}$ and $F(x) = 0$ for $x \in \mathbb{R}_{< 0}$ and $F^{-1} = \ln(1+p)-\ln(1-p)$ for $p \in [0,1)$ satisfies
the Lipschitz condition from the Lipschitz mechanism.

\item Theorem~\ref{lem:lipschitzlogistic}: The standard Logistic distribution with $F(x) = \frac{1}{1+\exp(-x)}$ for $x \in \mathbb{R}$ and $F^{-1} = -\ln(\, p/(1-p)\, )$ for $p \in [0,1)$ satisfies
the Lipschitz condition from the Lipschitz mechanism.

\end{itemize}

For some of these distributions the Lipschitz instantiates popular mechanisms from the literature:

\begin{itemize}

\item Theorem~\ref{lem:gumbelem}: The Lipschitz mechanism with $F^{-1}(p) = -\log(-\log(p))$ from the standard Gumbel distribution instantiates \REVISE{}{the }exponential mechanism (for $\kappa = 1$)
and peeling technique \cite{durfee2019practical} (for $\kappa = k$).

\item Theorem~\ref{lem:laplacernm}: The Lipschitz mechanism with $F^{-1}(p) = \sgn(p-\frac{1}{2}) \ln(1-|2p-1|)$ from \REVISE{}{the }standard Laplace distribution instantiates/matches report-noisy max mechanism \cite{dwork2014algorithmic} (for $\kappa = 1$)
and oneshot Laplace mechanism \cite{qiao2021oneshot} (for $\kappa = k$).
\item Theorem~\ref{lem:exponentialpf}: The Lipschitz mechanism with $F^{-1}(p) = -\log(1-p)$ from \REVISE{}{the }standard Exponential distribution instantiates permute-and-flip \cite{mckenna2020permute}.

\end{itemize}

\subsection{Exponential Lipschitz Mechanism: Permute-And-Flip Mechanism}

The Lipschitz mechanism is $\eps$-DP when adding exponentially distributed noise:

\begin{theorem}[Lipschitz condition: Exponential distribution] \label{lem:lipschitzexponential}
Let $F(x) = 1-\exp(-x)$ and $x,c \in \mathbb{R}$.

\REVISE{Then $$| \log(1-F(x))-\log(1-F(x+c))| \le c$$.}{Then $| \log(1-F(x))-\log(1-F(x+c))| \le |c|$.}
\end{theorem}
\begin{proof}
\begin{align*}
\exp(c) &= {\exp(-x-(-x-c) )}  \\
\exp(c) &= \frac{\exp(-x)}{\exp(-x-c)}  \\
\exp(c) &= \frac{1-(1-\exp(-x))}{1-(1-\exp(-x-c))}  \\
\exp(c) &= \frac{1-F(x)}{1-F(x+c)}  \\
\end{align*}
From the last equality we obtain: 
$$|\ln(1-F(x))-\ln(1-F(x+c))|\leq c.$$
\end{proof}

The Lipschitz condition follows for $c \in \mathbb{R}$ if it is met for $c \ge 0$:

\begin{lemma} \label{lem:lipschitzmirror}
Let $F(x)$ be a strictly increasing function.

If $\frac{1-F(x)}{1-F(x+c)} \le \exp(c)$ for $c \ge 0$, then for $x,c \in \mathbb{R}$ follows:

$$| \log(1-F(x))-\log(1-F(x+c))| \le \REVISE{c}{|c|}$$
\end{lemma}
\begin{proof}

\textbf{case $c \ge 0$:} As $F(x)$ is strictly increasing it follows that $\exp(-c) \le 1 \le \frac{1-F(x)}{1-F(x+c)} $. Also, from the statement we know that 
$\frac{1-F(x)}{1-F(x+c)} \le \exp(c)$ for $c\geq 0$. Thus, 
$\exp(c)\leq \frac{1-F(x)}{1-F(x+c)}\leq \exp(c)$, which implies 
$$| \log(1-F(x))-\log(1-F(x+c))| \le \REVISE{c}{|c|}.$$

\textbf{case $c < 0$:} Let $c' = -c \ge 0$ and $x' = x+c$.

\begin{align*}
1 \le \frac{1-F(x')}{1-F(x'+c')} & \le \exp(c') \\
1 \le \frac{1-F(x+c)}{1-F(x+c-c)} & \le \exp(-c) \\
1 \le \frac{1-F(x+c)}{1-F(x)} & \le \exp(-c) \\
\exp(c) \le \frac{1-F(x)}{1-F(x+c)} & \le 1 \le \exp(-c) \\
\end{align*}
from which we obtain 
$$c\le \ln(1-F(x))-\ln(1-F(x+c)) \le -c$$
and $| \log(1-F(x))-\log(1-F(x+c))| \le \REVISE{c}{|c|}.$

\end{proof}

The Lipschitz mechanism with $k = 1$ instantiates the Permute-And-Flip mechanism \cite{mckenna2020permute}  when adding exponentially distributed noise (confirming the results of \cite{ding2021permute}):

\begin{theorem}[Permute-and-Flip via Exponential Noise] \label{lem:exponentialpf}
If $F^{-1}(p) = -\ln(1-p)$ then $Pr[Y = y]$ matches the Permute-and-Flip mechanism.
\end{theorem}
\begin{proof}

Let $q_y = -\frac{\eps}{2\Delta} \loss(y)$ for a $\loss$ function with sensitivity $\Delta$.

\begin{align*}
Y = &\argmax_{\y \in \mathbb{Y}} \{ q_\y-\ln(1-U_\y) \}\\
 = &\argmax_{\y \in \mathbb{Y}}  \{ \exp(q_\y-\ln(1-U_\y)) \}\\
 = &\argmax_{\y \in \mathbb{Y}} \{ \frac{\exp(q_\y)}{(1-U_\y))} \}\\
 = &\argmin_{\y \in \mathbb{Y}} \{ {(1-U_\y)}\exp(-q_\y) \}\\
\end{align*}

This means for each $\y$ a (uniform) random number $R_\y$ between $0$ and $\exp(-q_\y)$ is drawn and the smallest is selected. Let $q_{*} = \max_{\y \in \mathbb{Y}} q_\y$. Then any $R_\y > \exp(-q_{*})$ will be certainly rejected and all non-rejected ones have the same probability of being the smallest. Let $N$ be the number of accepted elements, such that $N = \sum_{z \in \mathcal{Y} \setminus \{y\}} B_z$ is a Poisson Binomial random variable where each summation term $B_z$ is a Bernoulli random variable with success probability $1-Pr[Y \neq z]$. Then $Pr[Y = y | N] = \frac{1-Pr[Y \neq y]}{N+1}$. The rejection probability $Pr[Y \neq y]$ is then:

\begin{align*}
Pr[ Y \neq y ) =& Pr[R_\y > \exp(-q_{*})] \\
 =& Pr[ (1-U_\y)\exp(-q_\y) > \exp(-q_{*})] \\
 =& Pr[ (1-U_\y) > \exp(q_\y-q_{*})] \\
 =& Pr[ U_\y < 1-\exp(q_\y-q_{*})] \\
 =&  1-\exp( q_\y-q_{*}) \\
\end{align*}

Selecting a random non-rejected item is equivalent to selecting the first non-rejected item if items are in a random order. This matches Permute-and-Flip, which goes through items in random order, rejects each item $y$ with probability  $1-\exp( q_\y-q_{*})$ and then selects the first item that does not get rejected.

\end{proof}

\subsection{Gumbel Lipschitz Mechanism: Exponential Mechanism}

The Lipschitz mechanism is $\eps$-DP when adding Gumbel distributed noise terms $F^{-1}(p) = -\log(1-p)$:

\begin{theorem}[Lipschitz condition: Gumbel distribution] \label{lipschitzgumbel}
Let $F(x) = \exp(-\exp(-x))$ and $x,c \in \mathbb{R}$. Then:

$$| \log(1-F(x))-\log(1-F(x+c))| \le \REVISE{c}{|c|}$$
\end{theorem}
\begin{proof}

Due to Lemma~\ref{lem:lipschitzmirror} one can presume without loss of generality that $c \ge 0$.

Let $\alpha = 1-F(x) \in [0,1]$ and $k = \exp(c) \ge 1$, then from Lemma~\ref{lem:bonferroni} follows that:

\begin{align*}
\frac{1-(1-\alpha)^{1/k}}{\alpha/k} & \ge 1 \\
\frac{\alpha/k}{1-(1-\alpha)^{1/k}} & \le 1 \\
\frac{\alpha}{1-(1-\alpha)^{1/k}} & \le k \\
\frac{1-F(x)}{1-F(x)^{\exp(-c)}} & \le  \exp(c) \\
\frac{1-F(x)}{1-\exp(-\exp(-x))^{\exp(-c)}} & \le  \exp(c) \\
\frac{1-F(x)}{1-\exp(-\exp(-x) \exp(-c))} & \le  \exp(c) \\
\frac{1-F(x)}{1-\exp(-\exp(-x-c))} & \le  \exp(c) \\
\frac{1-F(x)}{1-F(x+c)} & \le  \exp(c) \\
\end{align*}

The claim follows due to Lemma~\ref{lem:lipschitzmirror}.

\end{proof}

The Lipschitz mechanism with $\kappa = k = 1$ instantiates the Exponential Mechanism \cite{mcsherry2007mechanism} when adding the Gumbel distributed noise terms $F^{-1}(p) = -\log(-\log(p))$ (for $\kappa = k > 1$ it matches the Peeling technique using the Exponential Mechanism \cite{durfee2019practical}):

\begin{theorem}[Exponential Mechanism via Gumbel trick]  \label{lem:gumbelem}
If $F^{-1}(p) = -\ln(-\ln(p))$ for the Lipschitz mechanism with $k = 1$, then $Pr[Y = y] \propto \exp( \frac{-\loss(\y)}{\eps/2\Delta} )$.
\end{theorem}
\begin{proof}

\REVISE{

Let $q_y =  \frac{-\loss(\y)}{\eps/2\Delta}$.
\begin{align*}
Y = &\argmax_{\y \in \mathbb{Y}} \{ \REVISE{q_y}{\log(\lambda_y)}-\ln( -\ln(U_\y) ) \} \\
=&\argmax_{\y \in \mathbb{Y}} \{ \exp( \REVISE{q_y}{\log(\lambda_y)}-\ln( -\ln(U_\y) ) ) \} \\
=&\argmax_{\y \in \mathbb{Y}} \{ \exp( \REVISE{q_y}{\log(\lambda_y)} ) / -\ln(U_\y) \} \\
=&\argmin_{\y \in \mathbb{Y}} \{ -\ln(U_\y) / \exp( \REVISE{q_y}{\log(\lambda_y)} )  \} \\
=&\argmin_{\y \in \mathbb{Y}} \{ -\ln(U_\y) / \lambda_\y  \} \\
\end{align*}
}{
Let $\mathbb{Y}$ be the selection domain and $\lambda_y = \exp(\frac{-\loss(\y)}{\eps/2\Delta})$ for any $y \in \mathbb{Y}$. Then:
\begin{align*}
Y = &\argmax_{\y \in \mathbb{Y}} \{ \frac{-\loss(\y)}{\eps/2\Delta}-\ln( -\ln(U_\y) ) \} \\
= &\argmax_{\y \in \mathbb{Y}} \{ {\log(\lambda_y)}-\ln( -\ln(U_\y) ) \} \\
=&\argmax_{\y \in \mathbb{Y}} \{ \exp( {\log(\lambda_y)}-\ln( -\ln(U_\y) ) ) \} \\
=&\argmax_{\y \in \mathbb{Y}} \{ \exp( {\log(\lambda_y)} ) / -\ln(U_\y) \} \\
=&\argmin_{\y \in \mathbb{Y}} \{ -\ln(U_\y) / \exp( {\log(\lambda_y)} )  \} \\
=&\argmin_{\y \in \mathbb{Y}} \{ -\ln(U_\y) / \lambda_\y  \} \\
\end{align*}
}

\REVISE{
It then follows that $Pr[Y = y] \propto \exp(q_y)$ as each $\argmin$ candidate is an exponential random variable with rate $\lambda_\y = \exp( q_\y )$. A proof for this property can be found in the following Lemma~\ref{lem:expclocks}.}{

Each $\argmin$ term $-\ln(U_\y) / \lambda_\y$ is then an exponential random variable with rate $\lambda_\y = \exp(\frac{-\loss(\y)}{\eps/2\Delta})$ and $Pr[Y = y] \propto \lambda_y$. A proof for this property can be found in the following Lemma~\ref{lem:expclocks}.}{

}
\end{proof}

For independent events with exponentially distributed time delays, each event's probability of preceding the others is proportional to their rate. This well-known property has for instance been used to prove the Gumbel Trick \cite{balog2017lost}:

\begin{lemma}[Exponential clocks] \label{lem:expclocks}
Let $U_\y$ be i.i.d. $U_\y \sim Unif(0,1)$ and $\mathbb{Y}$ be a finite set. Let $Y =\argmin_{\y \in \mathbb{Y}} \{ -\ln(U_\y) / \lambda_\y  \}$ supported over $\mathbb{Y}$. Then $Pr[Y = y] \propto \lambda_\y$.
\end{lemma}
\begin{proof}
Let $E_\y =-\ln(U_\y) / \lambda_\y$. 

As this is how one would generate an exponential random variable with rate $\lambda_y$, it follows that the density $Pr[E_y = x]= \lambda_y \exp(-x \lambda_y)$, cumulative $Pr[E_y \le x] = 1-\exp(-\lambda_y x)$ and complementary cumulative $Pr[E_y > x] = \exp(-\lambda_y x)$. The probability from the claim can then be written as:
\begin{align*}
 & Pr[Y = y] = Pr[E_y = \min_{z \in \mathcal{Y}} E_z] \\
= & \int_{0}^\infty Pr[E_y = x] \prod_{z \in \mathcal{Y} \setminus \{y\}} Pr[E_z > x] \, dx \\
= & \int_{0}^\infty \lambda_y \exp(-x\lambda_y) \prod_{z \in \mathcal{Y} \setminus \{y\}}  \exp(-x\lambda_z) \, dx\\
= &  \int_{0}^\infty \lambda_y \prod_{z \in \mathcal{Y}}  \exp(-x\lambda_z) \, dx\\
= &  \lambda_y \int_{0}^\infty \exp(-x \sum_{z \in \mathcal{Y}} \lambda_z ) \, dx\\
= & \lambda_y / \sum_{z \in \mathcal{Y}} \lambda_z
\end{align*}

The last equality follows from $\int a\exp(-bx) dx = -\frac{a}{b} \exp(-bx) + c$ for any $a,b,c \in \mathbb{R}$ with $b > 0$. Let $a = \lambda_y$ and $b = \sum_{z \in \mathcal{Y}} \lambda_z$. Then $\int_0^\infty a\exp(-bx) dx = (\lim_{x \rightarrow \infty} -\frac{a}{b} \exp(-bx) + c)-  (-\frac{a}{b} \exp( -b  \cdot 0) + c) = c-(-\frac{a}{b} +c) = \frac{a}{b} = \lambda_y / \sum_{z \in \mathcal{Y}} \lambda_z$.  Hence, $Pr[Y = y] \propto \lambda_\y$.

\end{proof}

\subsection{Laplace Lipschitz Mechanism: Report Noisy Max Mechanism}

\begin{theorem} \label{lem:laplacernm}
Trivially, the Lipschitz mechanism with $k = 1$ instantiates the Report Noisy Max mechanism \cite{dwork2014algorithmic} and the Oneshot Laplace Mechanism \cite{qiao2021oneshot} when adding Laplace distributed noise.
\end{theorem}
As the Lipschitz condition limits how fast a distribution function can change, satisfying the Lipschitz condition is inherited by doubled/mirrored distribution:

\begin{lemma} \label{lem:lipschitzdouble}
Let $F(x)$ be a strictly increasing function that satisfies the Lipschitz condition $|\log{(1-F(x))}-\log{(1-F(x+c))}| \le \REVISE{c}{|c|}$.

Let $F_{double}(x) = \begin{cases} \frac{F(-x)}{2} & x < 0 \\ \frac{1}{2}+\frac{F(x)}{2} & x \ge 0 \end{cases}$.

Then $F_{double}(x)$ also satisfies the Lipschitz condition

$|\log{(1-F_{double}(x))}-\log{(1-F_{double}(x+c))}| \le \REVISE{c}{|c|}$.

\end{lemma}
\begin{proof}

Due to Lemma~\ref{lem:lipschitzmirror} one can presume without loss of generality that $c \ge 0$.

Case $x \ge 0$ (which implies $x+c \ge 0$) where $F_{double}(x) = \frac{1}{2}+\frac{F(x)}{2}$ :

\begin{align*}
\frac{ 1-F(x) }{ 1-F(x+c) } & \le \exp(c) \\
\frac{ \frac{1}{2}-\frac{F(x)}{2} }{ \frac{1}{2}-\frac{F(x+c)}{2} } &\le \exp(c)\\
\frac{ 1-(\frac{1}{2}+\frac{F(x)}{2}) }{ 1-(\frac{1}{2}+\frac{F(x+c)}{2}) } &\le \exp(c)\\
\frac{1-F_{double}(x)}{1-F_{double}(x+c)} &\le \exp(c) \\
\end{align*}
The latter implies 

$$\log(1-F_{double}(x))-\log(1-F_{double}(x+c))\le \REVISE{c}{|c|}.$$ 
Furthermore, for $x\geq 0$, $F_{double}(x)$ is strictly increasing. Thus,  
$\log(1-F_{double}(x))\geq \log(1-F_{double}(x+c))$ and $F_{double}(x)$ satisfies the Lipschitz condition. 

Case $x < 0$ and $x+c < 0$ where $F_{double}(x) = \frac{F(-x)}{2}$.

\begin{align*}
\exp(-c) \le \frac{ 1-{F(-x)} }{ 1-{F(-x-c)} } & \le  \exp(c) \\
 \exp(-c) \le \frac{ 2-{F(-x)} }{ 2-{F(-x-c)} } & \le  \exp(c) \\
 \exp(-c) \le \frac{ 1-{F(-x)/2} }{ 1-{F(-x-c)/2} } & \le  \exp(c) \\
 \exp(-c) \le \frac{1-F_{double}(x)}{1-F_{double}(x+c)} & \le  \exp(c) \\
\end{align*}

In the second step it is exploited that adding the same positive value to numerator and denominator can only move the ratio  closer to $1$ (see Lemma~\ref{lem:ratioclosertoone}).

Case $x < 0$ and $x+c \ge 0$:

In this case $1-F_{double}(x)$ is strictly increasing, whereas $1-F_{double}(x+c)$ is strictly decreasing with $x$. Thus, the ratio $\frac{1-F_{double}(x)}{1-F_{double}(x+c)}$ cannot be larger than for the cases $x = 0$ or $x+c = 0$ that have already been covered.
\end{proof}

The Lipschitz mechanism is $\eps$-DP when adding Laplace distributed noise:

\begin{theorem}[Lipschitz condition: Laplace distribution] \label{lem:lipschitzlaplace}

Let $F(x) = 1-\exp(-x)$ and $x,c \in \mathbb{R}$.

Let $F_{double}(x) = \begin{cases} \frac{F(-x)}{2} & x < 0 \\ \frac{1}{2}+\frac{F(x)}{2} & x \ge 0 \end{cases}$.

Then $| \log(1-F_{double}(x))-\log(1-F_{double}(x+c))| \le \REVISE{c}{|c|}$.

\end{theorem}
\begin{proof}
The exponential distribution satisfies the Lipschitz condition (see Theorem~\ref{lem:lipschitzexponential}) and the Laplace distribution satisfies  the Lipschitz condition by being the double exponential distribution (see Lemma~\ref{lem:lipschitzdouble}).
\end{proof}

\subsection{Logistic Lipschitz Mechanism}

The Lipschitz mechanism is $\eps$-DP when adding Halflogistic distributed noise:

\begin{theorem}[Lipschitz condition: Halflogistic distribution] \label{lem:lipschitzhalflogistic}
Let $F(x) = \frac{1-\exp(-x)}{1+\exp(-x)}$ and $x,c \in \mathbb{R}$. Then $| \log(1-F(x))-\log(1-F(x+c))| \le \REVISE{c}{|c|}$.
\end{theorem}
\begin{proof}

\begin{align*}
& \frac{1-F(x)}{1-F(x+c)} \\
=& \frac{1-\frac{1-\exp(-x)}{1+\exp(-x)} }{1 - \frac{1-\exp(-x-c)}{1+\exp(-x-c)} } \\
=& \frac{ \frac{1+\exp(-x)-1+\exp(-x)}{1+\exp(-x)} }{\frac{1+\exp(-x-c)-1+\exp(-x-c)}{1+\exp(-x-c)}} \\
=& \frac{ \frac{2\exp(-x)}{1+\exp(-x)} }{\frac{2\exp(-x-c)}{1+\exp(-x-c)}} \\
=& \frac{ \frac{\exp(-x)}{1+\exp(-x)} }{\frac{\exp(-x-c)}{1+\exp(-x-c)}} \\
=& \frac{{1+\exp(-x-c)}}{1+\exp(-x)} \cdot \frac{\exp(-x) }{\exp(-x-c)} \\
=& \frac{{\exp(-x)+\exp(-2x-c)}}{\exp(-x-c)+\exp(-2x-c)}\\
\end{align*}

The additive term $\exp(-2x-c) > 0$ in numerator and denominator only moves the ratio closer to $1$ (see Lemma~\ref{lem:ratioclosertoone}). As $\exp(-x)$ is the distribution function of the exponential distribution, the claim then follows via Theorem~\ref{lem:lipschitzexponential}.

\end{proof}

\begin{theorem}[Lipschitz condition: Logistic distribution] \label{lem:lipschitzlogistic}
Let $F(x) = \frac{1}{1+\exp(-x)}$ and $x,c \in \mathbb{R}$. Then $| \log(1-F(x))-\log(1-F(x+c))| \le \REVISE{c}{|c|}$.
\end{theorem}
\begin{proof}
Follows from Lemma~\ref{lem:lipschitzdouble} and Theorem~\ref{lem:lipschitzhalflogistic} as the Logistic distribution is defined as the ``double'' Halflogistic distribution.
\end{proof}
\clearpage

\subsection{Additional Theorems and Proofs} \label{sec:proofs}

\begin{lemma}[canonical loss function] \label{thm:canonicalloss}

Let $\mathbb{Y}$ be some discrete-valued output domain, $\y \in \mathbb{Y}$ 
and $\OPT^{-1}(y)$ comprise any score vectors $\vec{x}$ s.t. $y \in \OPT(\vec{x})$, where $\OPT(\vec{x})$ is the optimal $k$-subset. Let $\vec{x}$ have sensitivity $\Delta_x$.

Then the function $\loss$ defined in the following has sensitivity $\Delta_{\textsc{loss}} = 1$:

$$\loss(y \mid \vec{x}) =  \min_{\, \vec{v} \in \OPT^{-1}(y) \,} \|\vec{x}-\vec{v}\|_\infty$$

\end{lemma}
\begin{proof}
Let $C\in \mathbb{R}$. In the following, $\vec{u} \pm C$ is used as a shorthand for the subspace $\vec{u}+\vec{c}$ with $\vec{c} \in [-C,+C]^d$.

As $\vec{x}_j = f( j \mid \hat{x} )/\Delta_f$ for $j \in \{1, \ldots, d\}$, a single user can change each component of $\vec{x}$ by at most $1$. If $\vec{x}$ is replaced by $\vec{x} \pm 1$ then each term $||\vec{x}-\vec{v}||_\infty$ is replaced by $||\vec{x}\pm 1-\vec{v}||_\infty$. Based on the definition of the $L_\infty$ norm, it then follows that $| \, \, ||\vec{x}-\vec{v}||_\infty \, - \, ||\vec{x}\pm 1-\vec{v}||_\infty \, \, | \, \le 1$.  

Then the values $||\vec{x}-\vec{v}||_\infty$ for different $\vec{v} \in \OPT^{-1}(y)$ form a set over which a minimum is taken. If all values of a set change by at most $C$, then their extrema also change at most by $C$ (see Lemma~\ref{lem:extrema} in supplementary material). Hence the sensitivity of $\loss(\y \mid \vec{x})$ is equal to $\Delta_{\textsc{loss}} = 1$.

\end{proof}

In the context of shift-invariant selection mechanisms, one can apply the following shifting trick to obtain a reduced sensitivity analysis for counting-based functions and alike:
\begin{theorem}[asymmetric sensitivity] \label{thm:asymmetric}

Let $\hat{a}, \hat{b} \in \mathbb{X}$ with $\textsc{users}(\hat{a}) \subset \textsc{users}(\hat{b})$ and $|\textsc{users}(\hat{b})| = |\textsc{users}(\hat{a})|+1$.

Let $f : \mathbb{Y} \times \mathbb{X} \rightarrow \mathbb{R}$, $\Delta_1, \Delta_2 \in \mathbb{R}_{\ge 0}$ and ${f(y \mid \hat{a})}-{f(y \mid \hat{b})} \le \Delta_1$ and ${f(y \mid \hat{b})}-{f(y \mid \hat{a})} \le \Delta_2$ for any $y \in \mathbb{Y}$.
Then the function $g(y \mid \hat{x}) = f(y \mid \hat{x})+|\textsc{users}(\hat{x})| \frac{\Delta_1-\Delta_2}{2}$ has sensitivity $\frac{\Delta_1+\Delta_2}{2}$.
\end{theorem}
\begin{proof}

By definition $|\textsc{users}(\hat{b})|-|\textsc{users}(\hat{a})| = 1$. Thus, for any $y \in \mathbb{Y}$ it holds that ${g(y \mid \hat{b})}-{g(y \mid \hat{a})} = {f(y \mid \hat{b})}-{f(y \mid \hat{a})}+\frac{\Delta_1}{2}-\frac{\Delta_2}{2}\leq \frac{\Delta_1+\Delta_2}{2}$ {and $g(y \mid \hat{a})-g(y \mid \hat{b})\leq f(y \mid \hat{a})-f(y\mid \hat{b})-\frac{\Delta_1+\Delta_2}{2}\leq\frac{\Delta_1}{2}-\frac{\Delta_2}{2}\leq \frac{\Delta_1+\Delta_2}{2}$}. 
\end{proof}

For positive reals $a,b,c$ the ratio of $\frac{b+c}{a+c}$ is smaller than $\frac{b}{a}$, because $b+c$ and $a+c$ are closer to each other than $a$ and $b$:
\begin{lemma}\label{lem:ratioclosertoone}
Let $a,b,c$ be positive reals with $b > a$. Then:

$$1 \le \frac{b+c}{a+c} \le \frac{b}{a}$$

\end{lemma}
\begin{proof}

Assume that $\frac{b+c}{a+c}>\frac{b}{a}$. Thus, we get $(b+c)a>(a+c)b$, 
$ab+ca>ab+cb$, and $a>b$ which contradicts $b>a$ in the statement. Thus, $\frac{b+c}{a+c}\leq \frac{b}{a}$. From $b>a$, we get $\frac{b+c}{a+c}>1$ and thus $\frac{b+c}{a+c}\geq 1$. Therefore, we obtain  $$1 \le \frac{b+c}{a+c} \le \frac{b}{a}$$. 
\end{proof}

\REVISE{ }{
\begin{lemma}\label{lem:explim}
Let $n,x \in \mathbb{R}$ with $n \ge 1$ and $x \ge -n$. Then:
$$1+x \le {\left(1+\frac{x}{n} \right)}^n \le \exp(x)$$
\end{lemma}
\begin{proof}

A variant of Bernoulli's inequality is $1+n x' \le {(1+x')}^{n}$ for any reals $n \ge 1$ and $x' \ge -1$. Due to $x \ge -n$, one can pick $x' = \frac{x}{n} \ge -1$ to obtain $1+x \le {(1+\frac{x}{n})}^{n}$.


Then $1+x \le \exp(x)$ for any $x \in \mathbb{R}$ is a well-known inequality due to the following. The derivative of $f(n) = {\left(1+\frac{x}{n} \right)}^n$ with respect to $n$ is $f'(n) = {\left(1+\frac{x}{n} \right)}^{n-1}$ and $f'(n) \ge 0$ as $1+\frac{x}{n} \ge 0$ for $x \ge -n$. Thus, for $x \ge -n$ the function $f(n)$ is increasing and base case $f(1) = 1+x$. The rest follows from $\lim_{n \rightarrow \infty} (1+\frac{x}{n})= \exp(x)$ which is a well-known identity that offers one way to define the exponential function. It follows a simple proof using \LHopitalsrule{}:

\begin{align*}
\lim_{n \rightarrow \infty} (1+\frac{x}{n})^n= &\exp( \lim_{n \rightarrow \infty}  \ln(1+\frac{x}{n}) n) \\
= &\exp \left( \lim_{n \rightarrow \infty}  \frac{\ln(1+\frac{x}{n})}{1/n} \right)\\
= &\exp \left( \lim_{n \rightarrow \infty}  \frac{ \frac{d}{dn} \ln(1+\frac{x}{n})}{ \frac{d}{dn} (1/n) }\right) \\
= &\exp \left( \lim_{n \rightarrow \infty}  \frac{ x(-1/n^2) \frac{1}{1+\frac{x}{n}}}{ -1/n^2 } \right)\\
= &\exp \left( \lim_{n \rightarrow \infty}  \frac{x}{1+\frac{x}{n}}\right) \\
= &\exp \left( x \right) \\
\end{align*}
\end{proof}
}

\newcommand{\RATIO}{\REVISE{C_{\alpha,k}}{R}}
\begin{lemma}\label{lem:bonferroni}
\REVISE{
Let $\alpha \in [0,1]$, \REVISE{$k \in \mathbb{R}$ with $k \ge 1$}{$k \in \mathbb{N}$} \REVISE{and $\REVISE{C_{\alpha,k}}{r_{\alpha,k}}{} = \frac{1-(1-\alpha)^{1/k}}{\alpha/k}$}{}. 

$$ 1 < \REVISE{C_{\alpha,k}}{r_{\alpha,k}} \le {\ln\left(\frac{1}{1-\alpha}\right)}/{\alpha} \REVISE{}{\text{ for }r_{\alpha,k} = \frac{1-(1-\alpha)^{1/k}}{\alpha/k}}$$

which follows that $\REVISE{C_{\alpha,k}}{r_{\alpha,k}} < \ln(4)$ for $\alpha \le 0.5$ (and $\REVISE{C_{\alpha,k}}{r_{\alpha,k}} < 1.03$ for $\alpha \le 0.05$).}{

Let $\alpha \in [0,1], k \in \mathbb{N}$. Then:

$$ 1 < \frac{1-(1-\alpha)^{1/k}}{\alpha/k} \le {\ln\left(\frac{1}{1-\alpha}\right)}/{\alpha}$$

which follows that $r_{\alpha,k} = \frac{1-(1-\alpha)^{1/k}}{\alpha/k} < \ln(4)$ for $\alpha \le 0.5$ (and $r_{\alpha,k} < 1.06$ for $\alpha \le 0.1$).
}

\end{lemma}
\begin{proof}

This relates to Bonferroni and \Sidak{} corrections for family-wise error rates \REVISE{}{(FWER}) in hypothesis testing. The \Sidak{} correction $\alpha_s = 1-(1-\alpha)^{1/k}$ is exact in case of independence, i.e., ${(1-\alpha_s)}^k = 1-\alpha$, whereas the Bonferroni correction $\alpha_b = \alpha/k$ is in that case conservative, i.e., ${(1-\alpha_b)}^k > 1-\alpha$ (the success rate is unnecessarily large). This also means $\alpha_s > \alpha_b$ and their ratio $\RATIO{} = r_{\alpha,k} =  \frac{\alpha_s}{\alpha_b} > 1$. \REVISE{We can then obtain}{According to Lemma~\ref{lem:explim} for $x \in \mathbb{R}, k \in \mathbb{N}$ with $x \ge -k$ it holds that $(1-\frac{x}{k})^k \le \exp(-x)$. Thus, we can set $x = -\alpha \cdot \RATIO{} = -\alpha \frac{\alpha_s}{\alpha / k} = -k{\alpha_s} \ge -k $} to obtain:
\begin{align*}
(1-\frac{\alpha}{k} \RATIO{})^k &\le  \exp( -\alpha \cdot \RATIO{} ) \\
\end{align*}

Clearly, $\frac{\alpha}{k} \RATIO{} = \alpha_b \frac{\alpha_s}{\alpha_b} = \alpha_s$. And we know $1-\alpha = {{(1-\alpha_s)}^k}$. Therefore by inserting $\alpha_s = \frac{\alpha}{k} \RATIO{}$ we get $1-\alpha = {{(1-\frac{\alpha}{k} \RATIO{})}^k}$. Thus, we can continue with:

\begin{align*}
{(1-\alpha)} = (1-\frac{\alpha}{k} \RATIO{})^k &\le  \exp( -\alpha \cdot \RATIO{} ) \\
\log{(1-\alpha)} &\le  -\alpha \cdot \RATIO{} \\
 -\log{\frac{1}{(1-\alpha)}}/\alpha & \le - \RATIO{} \\
 \RATIO{} &\le \log{\frac{1}{(1-\alpha)}}/\alpha
\end{align*}
The rest follows from $\log{\frac{1}{(1-\alpha)}}/\alpha$ being strictly decreasing.
\end{proof}

\begin{lemma}[EM utility guarantees] \label{lem:emguarantees}
Let $\mathbb{Y}$ be the selection domain, \REVISE{$\eps$ be a positive real}{$\eps \in \mathbb{R}_{\ge 0}$}
and $\Delta$ be the sensitivity of the loss function $\loss(y \mid \hat{x})$.

Then iff $Y$ is a random variable supported over $\mathbb{Y}$ with $Pr[Y = y] \propto \exp( \frac{\eps \, \loss(y \mid \hat{x}) }{-2\Delta} )$, then with probability $1-\alpha$ it holds that $ \loss(Y) \le \loss(\OPT \mid \hat{x})+\mathcal{E}$ with:

$$\mathcal{E} = \frac{2\Delta }{\eps} \left( \ln\left(\frac{|\mathbb{Y}|}{|\OPT|}\right)+\ln\left(\frac{1}{\alpha}\right) \right)$$

where $\OPT$ are all selection options with minimal loss.
\end{lemma}
\begin{proof}

Let $u(y) = -\loss(y \mid \hat{x})$. Theorem 3.11 in \cite{dwork2013algorithmic}):
{\scriptsize
$$Pr[ u(Y) \le u(\OPT) -  \frac{2\Delta }{\eps} \left( \ln\left(\frac{|\mathbb{Y}|}{|\OPT|}\right)+ c \right) ] \le \exp(-c)$$

$$Pr[ u(Y) > u(\OPT) -  \frac{2\Delta }{\eps} \left( \ln\left(\frac{|\mathbb{Y}|}{|\OPT|}\right)+ c \right) ] \ge 1-\exp(-c)$$
}
With probability $1-\exp(-c)$:
{\scriptsize
$$u(Y) > u(\OPT) -  \frac{2\Delta }{\eps} \left( \ln\left(\frac{|\mathbb{Y}|}{|\OPT|}\right)+ c \right) $$

$$-\loss(Y \mid \hat{x}) > -\loss(\OPT \mid \hat{x}) -  \frac{2\Delta }{\eps} \left( \ln\left(\frac{|\mathbb{Y}|}{|\OPT|}\right)+ c \right) $$

$$\loss(Y \mid \hat{x}) < \loss(\OPT \mid \hat{x})+\frac{2\Delta }{\eps} \left( \ln\left(\frac{|\mathbb{Y}|}{|\OPT|}\right)+ c \right) $$
}
Let $c = \ln\left(\frac{1}{\alpha}\right)$, then with probability $1-\alpha$:
{\scriptsize
$$ \loss(Y \mid \hat{x}) < \loss(\OPT \mid \hat{x})+\frac{2\Delta }{\eps} \left( \ln\left(\frac{|\mathbb{Y}|}{|\OPT|}\right)+\ln\left(\frac{1}{\alpha}\right) \right)$$
}

\end{proof}

\begin{lemma}[$\CANONICAL$ utility loss bounds] \label{lem:canonutil}
Let $Y_1, \ldots, Y_k$ be the selected set by $\CANONICAL$ with $\gamma \in (0,1]$ and $F^{-1}$ from the Gumbel distribution with tail item $T = \argmin_{i \in \{Y_1, \ldots, Y_k \}} \top{i}$. With at least probability $1-\alpha$:

$$ \top{T} < \top{k}+\frac{2\Delta }{\gamma \eps} \left( k \ln(d/k)+\ln \frac{1}{\alpha}+ \ln(c_{d,k}) \right)$$

\end{lemma}
\begin{proof}
For $\gamma = 1$ it follows directly that the loss value of each subset is $\top{t}$, that the optimal loss is $-\loss(\OPT \mid \hat{x}) = \top{k}$ and the logarithm of the domain size is $\ln \mathbb{Y} = \ln {d \choose k} = k \ln(d/k)+ \ln(c_{d,k})$ with $1 < c_{d,k} =  { {d \choose k} }/{\frac{d^k}{k^k}} \le \exp(k)$. Due to Corollary~\ref{cor:superior} for $\gamma < 1$ the privacy loss $\eps$ must simply be replaced with $\gamma \eps$.
\end{proof}

\begin{lemma}[$\PEELING$ utility loss bounds] \label{lem:peelutil}
Let $Y_1, \ldots, Y_k$ be the selected set by $\PEELING$ with $T = \argmin_{j \in \{Y_1, \ldots, Y_k \}} \vec{x}_j$. Let $r_{\alpha,k} = \frac{1-(1-\alpha)^{1/k}}{\alpha/k}$.  Then with probability $(1-\alpha)$:

$$ \top{T} < \top{k}+\frac{2\Delta }{\eps} \left( k \ln(d k) +k \ln \frac{1}{\alpha}- k\ln(r_{\alpha,k}) \right)$$
\end{lemma}
\begin{proof}
Each selection has $|\mathbb{Y}| = d$ and $-\loss(\OPT) \ge \top{k}$. Let $\alpha' = 1-(1-\alpha)^{1/k}$. If each of the $k$ selections $Y_i$ satisfies $\loss(Y_i) < \top{k}+\frac{2\Delta \ln(d/\alpha') }{\eps}$ with probability $\alpha'$, then all items $Y_1, \ldots, Y_k$ satisfy $\loss(Y_i) < \top{k}+\frac{2\Delta \ln(d/\alpha) }{\eps}$ with probability $\alpha$, which includes the tail item $T \in \{ Y_1, \ldots, Y_k\}$.
By replacing $\ln \alpha'$ in $\loss(Y_i \mid \hat{x}) < \top{k}+\frac{2\Delta \ln(d/\alpha') }{\eps}$ with $\log(1-\sqrt[k]{1-\alpha}) = \log{(\alpha/k)}+\log(r_{\alpha,k})$ one then obtains the claim. 
\end{proof}

\begin{lemma} \label{lem:extrema}
Let $A,B,C \in \mathbb{R}$ and $(A \pm B)$ be a shorthand for the interval $[A-B, A+B]$. Let $a,b,\ldots,z \in \mathbb{R}$. Then if $a' \in (a \pm C), b' \in (b \pm C), \ldots, z' \in (z \pm C)$, it holds that:

$$\min \{a',b',\ldots, z'\} \in (\min\{a,b,\ldots, z\} \pm C)$$
$$\max \{a',b',\ldots, z'\} \in (\max\{a,b,\ldots, z\} \pm C)$$
\end{lemma}
\begin{proof}
Follows from $\min \{a'-1,\ldots, z'-1\} \le \min \{a',\ldots, z'\} \le \min \{a'+1,\ldots, z'+1\}$
and analogously $\max \{a'-1,\ldots, z'-1\} \le \max \{a',\ldots, z'\} \le \max \{a'+1,\ldots, z'+1\}$.

\end{proof}

\begin{lemma}[top-$k$ canonical loss function]  \label{lem:topkcanonical}

Let $y \in \mathcal{C}_{h,t}$ be a $k$-subset of $\{1, \ldots, d\}$.
$$\loss(y  \mid \vec{x}) = \min_{\, \vec{v} \in \OPT^{-1}(y) \,} {\|\vec{x}-\vec{v}\|_\infty} = \frac{\top{h+1}-\top{t}}{2}$$
\end{lemma}
\begin{proof}

If $t = k$, then $y \in \OPT(\vec{x})$ and $\loss(y  \mid \vec{x}) = 0$.

If $t > k$, then from $y \in \mathcal{C}_{h,t}$ follows that  $\{\topitem{1}, \ldots, \topitem{h} \} \subset y$, but the top-$k$ item $\topitem{h+1} \notin y$. In order for $y$ to become optimal all of its items need to catch up with the missing top-$k$ item $\topitem{h+1}$. The tail item $\top{t}$ has the largest gap to $\top{h+1}$. Let $g = \top{h+1}-\top{t}$ be the gap between $\top{h+1}$ and $y$'s tail $\top{t}$ that must become $0$ for $y$ to become an optimal solution. 

Let $\vec{a} \in \mathbb{R}^d$ with $\vec{a}_i = \begin{cases} 1 & i \in y \\ 0 & \text{otherwise} \end{cases}$.

Let $\vec{v} = \vec{x}+\frac{g}{2}\vec{a}-\frac{g}{2}(1-\vec{a})$, which increases any $\vec{x}_j$ with $j \in y$ by $\frac{g}{2}$ and decreases all others by $\frac{g}{2}$.
Through algebraic reformulations one gets $\vec{v} = \vec{x}+g \vec{a}-\frac{g}{2}$.

Then $y \in \OPT(\vec{v})$ for $\vec{v} = \vec{x}+\frac{g}{2}\vec{a}-\frac{g}{2}(1-\vec{a}) = \vec{x}+g \vec{a}-\frac{g}{2}$.

$$\loss(y  \mid \vec{x}) = \|\vec{x}-\vec{x}-g \vec{a}+\frac{g}{2}\|_\infty = {\|\frac{g}{2}-g \vec{a}\|_\infty}$$

Due to the definition of the $L_\infty$ norm ${\|\frac{g}{2}-g \vec{a}\|_\infty} = \frac{g}{2}$, because $\max(\frac{g}{2}-g \vec{a}) = \frac{g}{2}$ and $\min(\frac{g}{2}-g \vec{a}) = \frac{g}{2}-g$. Hence one gets:

$$\loss(y  \mid \vec{x}) = \frac{g}{2} = \frac{\top{h+1}-\top{t}}{2}$$
\end{proof}

\begin{lemma} \label{lemma:jointcanonical}
Let $k,d \in \mathbb{N}$ with $k < d$. 
Let $\vec{x} \in \mathbb{R}^d$ and $y = \{y_1, \ldots, y_k\} \subset \{1, \ldots, d\}$. 
Let $\vec{y} \in \mathbb{R}^k$ and $\vec{y}_\ell = \vec{x}_{y_\ell}$ for $\ell \in \{1, \ldots, k\}$. Let $j_1, \ldots, j_d$ be indices $\{1, \ldots, d\}$ sorted by $\vec{x}$ and $i_1, \ldots, i_k$ be indices $\{1, \ldots, k\}$ sorted by $\vec{y}$ such that :
\REVISE{
$$\vec{x}_{j_1} \ge \vec{x}_{j_2} \ldots \ge  \vec{x}_{j_d}$$
$$\vec{y}_{i_1} \ge \vec{y}_{i_2} \ldots \ge  \vec{y}_{i_d}$$
}{
$$\vec{x}_{j_1} \ge \vec{x}_{j_2} \ge \ldots \ge  \vec{x}_{j_d} \text{ and } \vec{y}_{i_1} \ge \vec{y}_{i_2} \ge \ldots \ge  \vec{y}_{i_d}$$
}

Let $\vec{x}_{[\ell]} = \vec{x_{j_\ell}}$ for any $\ell \in \{1, \ldots, d\}$ and $\vec{y}_{[\ell]} = \vec{y_{i_\ell}}$ for any $\ell \in \{1, \ldots, k\}$.
Let $\mathcal{C}_{h,t}$ for any $h \in \{0, \ldots, k-1\}$ and $t \in \{k, \ldots, d\}$ be defined as in Definition~\ref{def:subsetclass}, which implies $\{j_1, \ldots, j_h\} \subset y$ and $y_{[k]} = t$. Let:

\begin{align*}
\textsc{canonical}( y \mid \vec{x} ) &= \frac{\vec{x}_{[h]}-\vec{y}_{[t]}}{2} \text{ for }y \in \mathcal{C}_{h,t} \\
\textsc{joint}( y \mid \vec{x} )&= \max_{\ell \in \{1, \ldots, k\} } \frac{\vec{x}_{[\ell]}-\vec{y}_{[\ell]}}{2} \\
\end{align*}

Let $OPT^{-1}(y)$ be a $d$-dimensional vector space where the indices $y$ have the $k$ largest values of each vector (allowing for ties), then:

(i) $\forall \, \vec{x} \in \mathbb{R}^d, y = \{y_1, \ldots, y_k\} \subset \{1, \ldots, d\}$:
$$\textsc{canonical}(\REVISE{\vec{x}}{y \mid \vec{x}}) = \min_{\vec{z} \in OPT^{-1}(y) } \| \vec{x}-\vec{z} \|_\infty$$

(ii) $\exists \, \vec{x} \in \mathbb{R}^d, y = \{y_1, \ldots, y_k\} \subset \{1, \ldots, d\}$: 
$$\textsc{joint}(\REVISE{\vec{x}}{y \mid \vec{x}}) < \textsc{canonical}(\REVISE{\vec{x}}{y \mid \vec{x}})$$
\end{lemma}
\begin{proof}

Claim (i) follows directly from Lemma~\ref{lem:topkcanonical} where $\textsc{canonical}$ is \REVISE{equivalent to}{matches} the subset loss function $\loss(y | \vec{x})$.
Claim (ii) follows from the following example.

Let $\vec{x} = [1,2,3,4,5,6,7,8,9,10]^T$ and $y = \{1,5,10\}$. As the index $9$ is missing from $y$, each vector $\vec{z} \in \OPT^{-1}(y)$ cannot have a larger value for the component with index $9$ than for the component with index $1$, i.e., $\vec{z}_9 \le \vec{z}_1$. Then:
\begin{align*}
\min_{\vec{z} \in OPT^{-1}(y) } \| \vec{x}-\vec{z} \| &=  \frac{9-1}{2} = \frac{8}{2} = 4 \\
\textsc{canonical}(y \mid \vec{x}) &= \frac{9-1}{2} = \frac{8}{2} = 4  \\
\textsc{joint}(y \mid \vec{x}) &= \frac{\max\{10-10, 9-5, 8-1\}}{2} = \frac{7}{2} \\
\end{align*}
Intuitively, $\frac{7}{2}$ is the maximal \REVISE{additive change}{change} to the scores $\{1,5,10\}$ to make them \REVISE{equally good}{ as good } as $\{8,9,10\}$, i.e., $\{1+\frac{7}{2},5+\frac{7}{2},10\} = \{4.5, 8.5, 10 \}$ is \REVISE{as good as}{not worse than} $\{8-\frac{7}{2}, 9-\frac{7}{2}, 10\} = \{4.5, 5.5, 10\}$. \REVISE{The reason this does not match}{It does not match} $\min_{\vec{z} \in OPT^{-1}(y) } \| \vec{x}-\vec{z} \|$ \REVISE{is that}{, because} raising $\vec{x}$ in all components of $y$ by $\frac{7}{2}$ and decreasing all others by $\frac{7}{2}$ will not produce a vector in $\OPT^{-1}(y)$, because $9-\frac{7}{2} = 5.5$ is still larger than $1+\frac{7}{2} = 4.5$ and that index is not featured in $y$. In contrast, $9-4 = 5$ is not larger than $1+4 = 5$.
\REVISE{
Note that the factors $\frac{1}{2}$ are due to the $L_\infty$ norm. The vector $\argmin_{\vec{z} \in OPT^{-1}(y) } \| \vec{x}-\vec{z} \|_\infty$ cannot only have larger values than $\vec{x}$ for indices $y$, but also have smaller values for the remaining indices. This corresponds to the concept of sensitivity from Definition~\ref{def:sensitivity} where users can both raise and increase all scores by some bounded value, which for count-based functions can be halved in the context of selection mechanisms (cf. Theorem~\ref{thm:asymmetric}). Thus, for count-based functions the factor $\frac{1}{2}$ would be dropped from $\textsc{canonical}(y \mid \vec{x})$ and $\textsc{joint}(y \mid \vec{x})$ in that context, but the conclusions would remain the same. }{}
\end{proof}

\REVISE{}{

The function $\JOINT$ matches the right-hand-side of the equation in Lemma~5 of \cite{joseph2021joint}. The factors $\frac{1}{2}$ in the proof are due to the $L_\infty$ norm, i.e., $\argmin_{\vec{z} \in OPT^{-1}(y) } \| \vec{x}-\vec{z} \|_\infty$ cannot only have larger values than $\vec{x}$ for indices contained in $y$, but also smaller values for indices missing from $y$. This corresponds to \REVISE{the concept of sensitivity from Definition~\ref{def:sensitivity} where users can both raise and increase all scores by some bounded value}{users being able to both raise and lower all scores by the sensitivity value (cf. Definition~\ref{def:sensitivity})}, which for count-based functions can be halved in the context of selection mechanisms (cf. Theorem~\ref{thm:asymmetric}).
}

\newcommand{\sensitivity}[1]{\left\langle \, #1 \, \right\rangle_{x}}

\end{document}